\documentclass{article}

\PassOptionsToPackage{numbers,compress,sort}{natbib}

\usepackage[final]{neurips_2021}

\usepackage[utf8]{inputenc} 
\usepackage[T1]{fontenc}    
\usepackage{hyperref}       
\usepackage{url}            
\usepackage{booktabs}       
\usepackage{amsfonts}       
\usepackage{nicefrac}       
\usepackage{microtype}      
\usepackage{xcolor}         
\usepackage{amssymb}
\usepackage{gordon}
\usepackage{amsmath}
\usepackage{dsfont}
\usepackage{thm-restate}
\usepackage{amsthm}
\usepackage{thmtools}
\usepackage{graphicx}
\usepackage[ruled,vlined]{algorithm2e}
\usepackage{wrapfig}
\usepackage{enumitem}
\usepackage{colortbl}
\usepackage{bm}
\usepackage{gensymb}
\usepackage{cleveref}

\def \be {\begin{eqnarray}}
\def \en {\end{eqnarray}}
\def \bega {\begin{align}}
\def \enda {\end{align}}
\def \a {\alpha}
\def \d {\delta}
\def \e {\epsilon}
\def \hy {\hat{y}}
\def \Tt {\mathtt{T}}
\def \rt {\mathtt{r}}
\def \tr {\nonumber \\}
\def \hTc {\widehat{\Tc}}
\def \hSc {\widehat{\Sc}}
\def \Hf {\Gamma}
\def \Rf {\mathfrak{R}}
\def \noin {\noindent}
\def \one {\mathds{1}}

\newcommand{\argmin}{\mathop{\mathrm{argmin}}}
\newcommand{\argmax}{\mathop{\mathrm{argmax}}}

\newcommand{\conv}{\mathop{\mathrm{conv}}}
\newcommand{\id}{\mathop{\mathrm{id}}}

\newcommand{\TwoTransferReal}{\Tt_\Hf^\rt(\Sc, \Tc)}
\newcommand{\TV}{d_{\rm TV}(\Sc, \Tc)}

\newtheorem{thm}{Theorem}
\newtheorem{definition}[thm]{Definition}

\newtheorem{prop}[thm]{Proposition}
\newtheorem{eg}[thm]{Example}
\newtheorem{lem}[thm]{Lemma}
\newtheorem{cor}[thm]{Corollary}

\title{Quantifying and Improving Transferability in Domain Generalization}

\author{%
  Guojun Zhang 
  \\School of Computer Science\\
  University of Waterloo\\
  Vector Institute \\
  \texttt{guojun.zhang@uwaterloo.ca} 
   \And
   Han Zhao \\
   Department of Computer Science \\
   University of Illinois at Urbana-Champaign \\
   \texttt{hanzhao@illinois.edu} 
  \AND
  Yaoliang Yu
  \\School of Computer Science\\
  University of Waterloo\\
  Vector Institute \\
  \texttt{yaoliang.yu@uwaterloo.ca} 
  \And
  Pascal Poupart
  \\School of Computer Science\\
  University of Waterloo\\
  Vector Institute \\
  \texttt{ppoupart@uwaterloo.ca} 
}

\begin{document}

\maketitle

\begin{abstract}

Out-of-distribution generalization is one of the key challenges when transferring a model from the lab to the real world.  Existing efforts mostly focus on building invariant features among source and target domains. Based on invariant features, a high-performing classifier on source domains could hopefully behave equally well on a target domain. In other words, we hope the invariant features to be \emph{transferable}. However, in practice, there are no perfectly transferable features, and some algorithms seem to learn ``more transferable'' features than others. How can we understand and quantify such \emph{transferability}? In this paper, we formally define transferability that one can quantify and compute in domain generalization. We point out the difference and connection with common discrepancy measures between domains, such as total variation and Wasserstein distance. We then prove that our transferability can be estimated with enough samples and give a new upper bound for the target error based on our transferability. Empirically, we evaluate the transferability of the feature embeddings learned by existing algorithms for domain generalization. Surprisingly, we find that many algorithms are not quite learning transferable features, although few could still survive. In light of this, we propose a new algorithm for learning transferable features and test it over various benchmark datasets, including RotatedMNIST, PACS, Office-Home and WILDS-FMoW. Experimental results show that the proposed algorithm achieves consistent improvement over many state-of-the-art algorithms, corroborating our theoretical findings.\footnote{Code available at \url{https://github.com/Gordon-Guojun-Zhang/Transferability-NeurIPS2021}.}
\end{abstract}

\section{Introduction}
One of the cornerstone assumptions underlying the recent success of deep learning models is that the test data should share the same distribution as the training data. However, faced with ubiquitous distribution shifts in various real-world applications, such assumption hardly holds in practice. For example, a self-driving recognition system trained using data collected in the daytime may continually degrade its performance during nightfall. The system may also encounter weather or traffic conditions in a new city that never appear in the training set. In light of these potentially unseen scenarios, it is of paramount importance that the trained model can generalize \emph{Out-Of-Distribution} (OOD): even if the target domain is not exactly the same as the source domain(s), the learned model should hopefully behave robustly under slight distribution shift. 

To this end, one line of works focuses on learning the so-called \emph{invariant representations}~\citep{ganin2015unsupervised,zhao2018adversarial,zhang2019bridging,albuquerque2019generalizing}. At a colloquial level, the goal here is to learn feature embeddings that lead to indistinguishable feature distributions from different domains. In practice, both the feature embeddings and the domain discriminators are often parametrized by neural networks, leading to an adversarial game between these two. Furthermore, in order to avoid degenerate solutions, the learned features are required to be informative about the output variable as well. This is enforced by placing a predictor over the features and minimize the corresponding supervised loss simultaneously~\citep{long2017deep,tzeng2017adversarial,ganin2016domain,tachet2020domain}.

Another line of recent works aims to learn features that can induce \emph{invariant predictors}, first termed as the invariant risk minimization (IRM)~\citep{peters2016causal,arjovsky2019invariant} paradigm. Roughly speaking, the goal of IRM is to discover a feature embedding, upon which the optimal predictors, i.e., the Bayes predictor, are invariant across the training domains. Again, at the same time, the features should be informative about the output variable as well. However, the optimization problem of IRM is rather difficult, and several follow-up works have proposed different relaxations to the original formulation~\citep{ahuja2020invariant,krueger2020out}.

Despite being extensively studied, both theoretical~\citep{zhao2019learning,rosenfeld2020risks} and empirical~\citep{koh2020wilds,gulrajani2020search} works have shown the insufficiency of existing algorithms for domain generalization (DG). Methods based on invariant features ignore the potential shift in the marginal label distributions across domains~\citep{zhao2019learning} and the methods based on invariant predictors are not robust to covariate shift~\citep{krueger2020out}. Perhaps surprisingly, empirical works have shown that with proper data augmentation and careful model tuning, the very basic algorithm of empirical risk minimization (ERM) demonstrates superior performance on domain generalization over existing methods on benchmark image datasets~\citep{gulrajani2020search,koh2020wilds}. This sharp gap between theory and practice calls for a fundamental understanding of the following question:
\begin{quote}
\itshape
What kind of invariance should we look for, in order to ensure that a good model on source domains also achieves decent accuracy on a related target domain? 
\end{quote}
In this work we attempt to answer the above question by proposing a criterion for models to look at, dubbed as \emph{transferability}, which asks for an invariance of the \emph{excess risks} of a predictor across domains. Different from existing proposals of invariant features and invariant predictors, which seek to find feature embeddings that respectively induce invariant marginal and conditional distributions, our notion of transferability depends on the excess risk, hence it directly takes into account the joint distribution over both the features and the labels. We show how it can be used to naturally derive a new upper bound for the target error, and then we discuss how to estimate the transferability empirically with enough samples. Our definition also inspires a method that aims to find more transferable features via representation learning using adversarial training. 

\begin{wrapfigure}[12]{r}{0.27\textwidth}
\vspace*{-1.8em}
\begin{center}
\includegraphics[width=\linewidth]{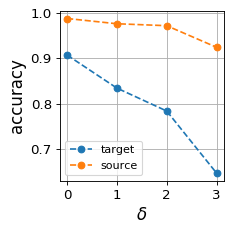}
\end{center}
\vspace{-5mm}
\caption{The target and source (test) accuracies of ERM on MNIST.}
\label{fig:demo}
\end{wrapfigure}
Empirically, we perform experiments to measure the transferability of several existing algorithms, on both small and large scale datasets. We show that many algorithms, including ERM, are not quite transferable under the definition (Fig.~\ref{fig:demo}, see more details in \S\ref{sec:exp}): when we go away from the optimal classifier (with distance $\delta$ in the parameter space), it could happen that the source accuracy remains high but the target accuracy drops significantly. This implies that during the training process, an existing algorithm may find a good source classifier with low target accuracy, hence violating the requirement for invariance of excess risks. In contrast, our algorithm is more transferable, and achieves consistent improvement over existing state-of-the-art algorithms, corroborating our findings.




\section{What is Transferability?}\label{sec:def_of_transfer}
In this section we present our definition of transferability in the classification setting. The setup of domain generalization is the following: 

\paragraph{Settings and Notation}  Given $n$ labeled source domains $\Sc_1, \dots, \Sc_n$, the problem of domain generalization is to learn a model from these source domains, in the hope that it performs well on an unseen target domain $\Tc$ that is ``similar'' to the source domains. Throughout the paper, we assume that both the source domains and the unseen target domain share the same input and output spaces, denoted as $\Xc$ and $\Yc$, respectively. For multi-class classification, the output space $\Yc = [K]$ is a set of labels for multi-class classification. For binary classification, we consider $\Yc = \{-1, +1\}$. Denote $\Hc$ as the hypothesis class. We define the classification error of a classifier $h \in \Hc$ on a domain $\Dc$ (or $\Sc$ for source domains, or $\Tc$ for target domains) as:\footnote{Throughout the paper, we will use the terms domain and distribution interchangeably.}
\begin{align}\label{eq:def_loss}
\epsilon_{\Dc}(h) = \Eb_{(x, y)\sim \Dc} [\ell(h(x), y)].
\end{align}
For $\ell(h(x), y) = \one(h(x) \neq y)$, where $\one(\cdot)$ is the usual indicator function, we use $\epsilon^{\rm 0-1}_{\Dc}(h)$ to denote it is the 0-1 loss.

In domain generalization, we often have several source domains. For the ease of presentation, we only consider a single source domain in this section, and later extend to the general case in Section~\ref{sec:exp}. Given two domains, the source domain $\Sc$ and the target domain $\Tc$, the task of domain generalization is to transfer a classifier $h$ that performs well on $\Sc$ to $\Tc$. We ask: how much of the success of $h$ on $\Sc$ can be transferred to $\Tc$? 

Note that in order to evaluate the transferability from $\Sc$ to $\Tc$, we need information from the target domain, similar to the test phase in traditional supervised learning. We believe a good criterion of transferability should satisfy the following properties:
\begin{enumerate}[topsep=0pt, parsep=0pt]
\item \emph{Quantifiable:} the notion should be quantifiable and can be computed in practice;
\item Any near-optimal source classifier should be near-optimal on the target domain. 
\item If the two domains are similar, as measured by e.g.,~total variation, then they are transferable to each other, but the converse may not be true. 
\end{enumerate}
At first glance the second criterion above might seem too strong and restrictive. However, we argue that in the task of domain generalization, we only have labeled source data and there is no clue to distinguish a classifier from another if both of them perform equally well on the source domain. Based on the second property, we first propose the following definition of transferability: 
\begin{definition}[\textbf{transferability}]\label{def:transfer}
$\Sc$ is $(\d_\Sc, \d_\Tc)_{\Hc}$-transferable to $\Tc$ if for $\d_\Sc > 0$, there exists $\d_\Tc > 0$ such that ${\argmin} (\e_\Sc, \d_\Sc)_\Hc \subseteq {}{\argmin} (\e_\Tc, \d_\Tc)_\Hc$, where:
\begin{align}
{\argmin} (\e_\Dc, \d_\Dc)_\Hc := \{h\in \Hc: \e_\Dc(h) \leq \inf_{h\in \Hc}\e_\Dc(h) + \d_\Dc\}. \nonumber
\end{align}
\end{definition}
\vspace{-0.5em}
In the literature the set ${\argmin}(\e_\Dc, \d_\Dc)_\Hc$ is also known as a \emph{$\d_\Dc$-minimal set} \citep{koltchinskii2010rademacher} of $\e_\Dc$, which represents the near-optimal set of classifiers. Note that the $\d$-minimal set depends on the hypothesis class $\Hc$. Throughout the paper, we omit the subscript $\Hc$ in the definition when there is no confusion. Def.~\ref{def:transfer} says that near-optimal source classifiers are also near-optimal target classifiers. Furthermore, it is easy to verify that our transferability is transitive: if $\Sc$ is $(\d_\Sc, \d_\Pc)$-transferable to $\Pc$, and $\Pc$ is $(\d_\Pc, \d_\Tc)$-transferable to $\Tc$, then $\Sc$ is $(\d_\Sc, \d_\Tc)$-transferable to $\Tc$. 

Next we define transfer measures, which we will show to be equivalent with Def.~\ref{def:transfer} in Prop.~\ref{prop:transfer_equiv}. 
\begin{definition}[\textbf{quantifiable transfer measures}]\label{def:transfer_measure}
Given some $\Hf \subseteq \Hc$, $\e_\Sc^* := \inf_{h\in \Hf} \e_\Sc(h)$ and $\e_\Tc^* := \inf_{h\in \Hf} \e_\Tc(h)$ we define the one-sided transfer measure, symmetric transfer measure and the realizable transfer measure respectively as:
\begin{align}\label{eq:one_sided_measure}
&\mathtt{T}_{\Hf}(\Sc\|\Tc) := \sup_{h\in \Hf}\e_\Tc(h) - \e_\Tc^* - (\e_\Sc(h) - \e_\Sc^*), \\
&\label{eq:symmetric_transfer}
\mathtt{T}_{\Hf}(\Sc, \Tc) := \max\{\Tt_\Hf(\Sc\|\Tc), \Tt_\Hf(\Tc\|\Sc)\} = \sup_{h\in \Hf}|\e_\Sc(h) - \e_\Sc^* - (\e_\Tc(h) - \e_\Tc^*)|, \\
&\label{eq:realizable_transfer}
\mathtt{T}^\rt_{\Hf}(\Sc, \Tc) := \sup_{h\in \Hf}|\e_\Sc(h) - \e_\Tc(h)|.
\end{align}
\end{definition}
\vspace{-0.5em}
\noin
The distinction between $\Hf$ and $\Hc$ will become apparent in Prop.~\ref{prop:transfer_equiv}.
Note that the one-sided transfer measure is not symmetric. If we want the two domains $\Sc$ and $\Tc$ to be mutually transferable to each other, we can use the symmetric transfer measure. We call both quantities as \emph{transfer measures}. Furthermore, the symmetric transfer measure reduces to \eqref{eq:realizable_transfer} in the realizable case when $\e_\Sc^* = \e_\Tc^* = 0$. In statistical learning theory, $\e_\Dc(h) - \e_\Dc^*$ is often known as an \emph{excess risk} \citep{koltchinskii2010rademacher}, which is the relative error compared to the optimal classifier. The transfer measures can thus be represented with the difference of excess risks. With Def.~\ref{def:transfer_measure}, we can immediately obtain the following result that upper bounds the target error:  
\begin{prop}[\textbf{target error bound}]\label{thm:target_bound}
Given $\Hf \subseteq \Hc$, for any $h\in \Hf$, the target error is bounded by:
\be\label{eq:target_error}
\e_\Tc(h) \leq \e_\Sc(h) + \e_{\Tc}^* - \e_\Sc^* + \Tt_{\Hf}(\Sc \| \Tc) \leq \e_\Sc(h) + \e_{\Tc}^* - \e_\Sc^* + \Tt_{\Hf}(\Sc, \Tc).
\en
\end{prop}
The first error bound of such type for a target domain uses $\Hc$-divergence \citep{ben2007analysis,blitzer2007learning,ben2010theory} for binary classification (or more rigorously, the $\Hc\Delta\Hc$-divergence). The main difference between ours and $\Hc$-divergence is that $\Hc$-divergence only concerns about the marginal input distributions, whereas the transfer measures depend on the joint distributions over both the inputs and the labels. We note that Proposition~\ref{thm:target_bound} is general and works in the multi-class case as well. Moreover, even in the binary classification case we can prove that our \Cref{thm:target_bound} is tighter than $\Hc$-divergence (see \Cref{prop:compare_trans_measure_hdiv} in the appendix). 

In practice we may not know the optimal errors. In this case, we can use the realizable transfer measure to upper bound the symmetric transfer measure (note that $\e_\Sc^*$ or $\e_\Tc^*$ may not be zero):

\begin{restatable}{prop}{UpperRealizable}\label{lem:upper_bound_realistic}
For $\Hf \subseteq \Hc$ and domains $\Sc$, $\Tc$ we have: $\Tt_\Hf(\Sc, \Tc) \leq 2\Tt_\Hf^\rt (\Sc, \Tc)$.
\end{restatable}

Since Def.~\ref{def:transfer} essentially asks that the excess risks of approximately optimal classifiers on the source domain are comparable between the source and target domains, we can show that Def.~\ref{def:transfer} and Def.~\ref{def:transfer_measure} are equivalent if $\Hf$ is a $\d$-minimal set:

\begin{restatable}[\textbf{equivalence between transferability and transfer measures}]{prop}{OneSide}
\label{prop:transfer_equiv}
Let $\d_\Sc > 0$ and $\Hf = \argmin(\e_\Sc, \d_\Sc)$ and suppose $\inf_{h\in \Hf} \e_\Tc(h) = \inf_{h\in \Hc} \e_\Tc(h)$. If $\mathtt{T}_{\Hf}(\Sc\|\Tc) \leq \d$ or $\Tt_\Hf(\Sc, \Tc) \leq \d$, then $\Sc$ is $(\d_\Sc, \d + \d_\Sc)$-transferable to $\Tc$. Furthermore, if $\Sc$ is $(\d_\Sc, \d_\Tc)$-transferable to $\Tc$, then $\mathtt{T}_{\Hf}(\Sc\|\Tc)\leq \d_\Tc$ and $\Tt_\Hf(\Sc, \Tc) \leq \max\{\d_\Sc, \d_\Tc\}$.  
\end{restatable}
In Prop.~\ref{prop:transfer_equiv}, we do not require $\Hf = \Hc$ since it is unnecessary to impose that all classifiers in $\Hc$ have similar excess risks on source and target domains. Instead, we only constrain $\Hf$ to be a $\d$-minimal set, i.e., $\Hf$ includes approximately optimal classifiers of $\Sc$. See also \Cref{eg:d_transfer_is_not_sim}. An additional assumption is that $\Hf$ also includes the optimal classifier of $\Tc$ which can be ensured by controlling $\d_\Sc$.

\subsection{Comparison with other discrepancy measures between domains}

In this subsection, we compare the realizable transfer measure~\eqref{eq:realizable_transfer} with other discrepancy measures between domains and focus on the 0-1 loss $\e^{\rm 0-1}_{\Dc}$.  We first note that $\Tt^{\mathtt{r}}_{\Hf}(\Sc, \Tc)$ can be written as an integral probability metric (IPM) \citep{sriperumbudur2009integral, muller1997integral}. The l.h.s.~of \eqref{eq:realizable_transfer} can be written as:
\begin{align}\label{eq:transfer_ipm}
&\Tt^{\mathtt{r}}_\Hf(\Sc, \Tc) := d_{\Fc_\Hf }(\Sc, \Tc), \mbox{ where } d_{\Fc}(\Sc, \Tc) = \sup_{f\in \Fc} \left|\sum_y \int f(x, y) (p_\Sc(x, y) - p_\Tc(x, y)) dx  \right|,
\end{align}
and $\Fc_{\rm \Hf} := \{(x, y)\mapsto \one(h(x)\neq y), h\in \Hf\}$. Typical IPMs~\citep{sriperumbudur2009integral} include MMD, Wasserstein distance, Dudley metric and the Kolmogorov--Smirnov distance (see Appendix \ref{app:other_ipm} for more details). However, $\Fc_\Hf$ is fundamentally different from these IPMs since it relies on an underlying function class $\Hf$. Our realizable transfer measure  shares some similarity with \citet{arora2017generalization}, where a changeable function class is used, but the exact choices of the function class are different. 

Even though the transferability can be written in terms of IPM, it is in fact a pseudo-metric:

\begin{restatable}[\textbf{pseudo-metric}]{prop}{pmetric}
For a general loss $\e_\Dc$ as in \eqref{eq:def_loss}, $\Tt^\rt_\Hf(\Sc, \Tc)$ is a pseudo-metric, i.e., for any distributions $\Sc, \Tc, \Pc$ on the same underlying space, we have $\Tt^\rt_\Hf(\Sc, \Sc) = 0$, $\Tt^\rt_\Hf(\Sc, \Tc) = \Tc^\rt_\Hf(\Tc, \Sc)$ (symmetry), and $\Tt^\rt_\Hf(\Sc, \Tc) \leq \Tt^\rt_\Hf(\Sc, \Pc) + \Tt^\rt_\Hf(\Pc, \Tc)$ (triangle inequality).
\end{restatable}
However in general $\Tt^\rt_{\Hf}(\Sc, \Tc)$ is not a metric since $\Tt^\rt_{\Hf}(\Sc, \Tc) = 0$ even if $\Sc \neq \Tc$. For instance, taking $\Hf = \{h^*\}$ to be the optimal classifier on both $\Sc$ and $\Tc$. we have $\Tt^\rt_\Hf (\Sc, \Tc)  = 0$, but $\Sc$ and $\Tc$ could differ a lot (see Figure \ref{fig:match_joint_suff_not_nece}). In the next result we discuss the connection between realizable transfer measures and total variation (c.f.~\Cref{app:other_ipm}). 

\begin{restatable}[\textbf{equivalence with total variation}]{prop}{upperTV}\label{prop:transfer_TV}
For binary classification with labels $\{-1, 1\}$, given the 0-1 loss $\e_\Dc = \e_\Dc^{\rm 0-1}$, we have $\Tt^{\mathtt{r}}_{\Hf}(\Sc, \Tc) \leq d_{\rm TV}(\Sc, \Tc)$ for domains $\Sc, \Tc$ and any $\Hf \subseteq \Hc$. Denote $\Hc_t$ to be the set of all binary classifiers. Then we have $d_{\rm TV}(\Sc, \Tc) \leq 4 \Tt_{\Hc_t}^\rt(\Sc, \Tc)$.
\end{restatable}

Prop.~\ref{prop:transfer_TV} tells us that transfer measures (see also Prop.~\ref{lem:upper_bound_realistic}) are no stronger than total variation, and in the realizable case, \eqref{eq:symmetric_transfer} is equivalent to the similarity of domains (as measured by total variation) if $\Hf$ is unconstrained. We can moreover show that transfer measures are strictly weaker, if we choose $\Hf$ to be some $\d$-minimal set:

\begin{wrapfigure}[20]{r}{0.5\textwidth}
\vspace{-10mm}
\begin{center}
\includegraphics[width=0.33\textwidth]{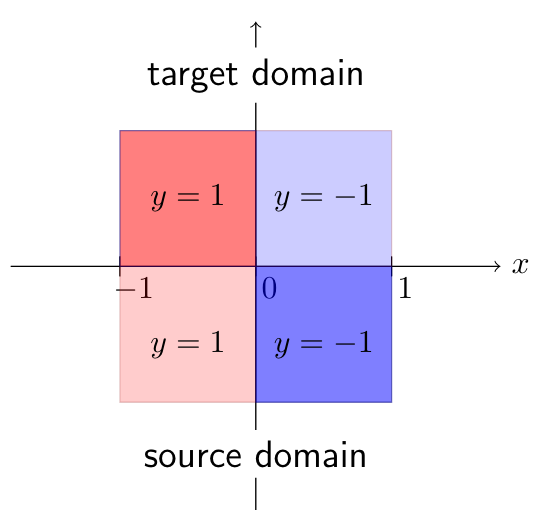}
\vspace{-2mm}
\caption{Visualization of Example \ref{eg:d_transfer_is_not_sim}. \textbf{Source domain}: $P_{\Sc}(Y=1, -1\leq X < 0) = 0.1$, $P_{\Sc}(Y=-1, 0\leq X < 1) = 0.9$. \textbf{Target domain}: $P_{\Tc}(Y=1, -1\leq X < 0) = 0.9$,  $P_{\Tc}(Y=-1, 0\leq X < 1) = 0.1$. The dark and light colors show the intensity of the probability mass. The vertical axis denotes whether it is the target or source domain (above or below $x$-axis).}
\label{fig:match_joint_suff_not_nece}
\end{center}
\end{wrapfigure}

\begin{eg}[\textbf{very dissimilar joint distributions but transferable}]\label{eg:d_transfer_is_not_sim}
We study the distributions described in Figure \ref{fig:match_joint_suff_not_nece}. The joint distributions are very dissimilar, i.e., for any $X, Y$ in the domain, $|p_\Sc(X, Y) - p_\Tc(X, Y)| = 0.8$. Define
\be\label{eq:step_func}
h_\rho(X) = \begin{cases}
1 & \textrm{ if } -1\leq X < \rho \\
-1 & \textrm{ if }\rho \leq X < 1
\end{cases}.
\en
We choose the hypothesis class $\Hc = \{h_\rho, \, \rho\in [-1, 1]\}$ and $\Hf = \{h_\rho, |\rho|\leq \delta/0.8\}$ (for small $\d$, say $\d < 0.01$) to be some neighborhood of the optimal source classifier $h^* = h_0$.
Then $\Tt_\Hf(\Sc, \Tc) = \sup_{h\in \Hf} |\e_\Sc(h) - \e_\Tc(h)| = \d$, and $\Sc$ is $(\d_\Sc, \d + \d_\Sc)$-transferable to $\Tc$ on $\Hf$ for any $\d_\Sc > 0$ according to Prop.~\ref{prop:transfer_equiv}. Note that $\e_\Sc^* = \e_\Tc^* = 0$. 
\end{eg}


\section{Computing Transferability}\label{sec:compute_transfer}
In the last section we proposed a new concept called transferability. However, although Def.~\ref{def:transfer} provides a theoretically sound result for transferability, it is hard to verify it in practice, since we cannot exhaust all approximately good classifiers, especially for rich models such as deep neural networks. Nevertheless, Prop.~\ref{thm:target_bound} and Prop.~\ref{prop:transfer_equiv} provide a framework to compute transferability through transfer measures, despite their simplicity. In this section we discuss how to compute these quantities by making necessary approximations based on transfer measures. There are two difficulties we need to overcome:
{\bf (1)} In practice we only have finite samples drawn from true distributions; 
{\bf (2)} We need a surrogate loss such as cross entropy for training and the 0-1 loss for evaluation. 
In \S\ref{sec:estimate} we show that our transfer measures can be estimated with enough samples, and in \S\ref{sec:surrogate} we discuss transferability with a surrogate loss. These results will be used in our algorithms in the next section.  

\subsection{Estimation of transferability}\label{sec:estimate}
We show how to estimate the transfer measure $\Tt_\Hf(\Sc\|\Tc)$ from finite samples. Other versions of transfer measures in Def.~\ref{def:transfer} follow analogously (see Appendix \ref{app:proofs} for more details). 

\begin{restatable}[\textbf{reduction of estimation error}]{lem}{ReductionEst}\label{lem:transfer_to_normal_error}
Given general loss $\e_\Dc$ as in \eqref{eq:def_loss}, suppose $\widehat{\Sc}$ and $\widehat{\Tc}$ are i.i.d.~sample distributions drawn from distributions of $\Sc$ and $\Tc$, then for any $\Hf\subseteq \Hc$ we have:
\begin{align}
&\Tt_\Hf(\Sc\|\Tc) \leq \Tt_\Hf(\widehat{\Sc}\| \widehat{\Tc})+ 2{\rm est}_{\Hf}(\Sc) + 2{\rm est}_{\Hf}(\Tc), \nonumber
\end{align}
with the estimation errors ${\rm est}_{\Hf}(\Sc) = \sup_{h\in \Hf} |\e_\Sc(h) - \e_{\widehat{\Sc}}(h)|,\, {\rm est}_{\Hf}(\Tc) = \sup_{h\in \Hf} |\e_\Tc(h) - \e_{\widehat{\Tc}}(h)|.$
\end{restatable}

This lemma tells us that estimating transferability is no harder than computing the estimation errors of both domains. If the function class $\Hf$ has uniform convergence property \citep{shalev2014understanding}, then we can guarantee efficient estimation of transferability. We first bound the sample complexity through Rademacher complexity, which is a standard tool in bounding estimation errors \citep{bartlett2002rademacher}:

\begin{restatable}[\textbf{estimation error with Rademacher complexity}]{thm}{RadeEst}\label{thm:estimate_transfer_measures}
Given the 0-1 loss $\e_\Dc = \e_\Dc^{\rm 0-1}$, suppose $\widehat{\Sc}$ and $\widehat{\Tc}$ are sample sets with $m$ and $k$ samples drawn i.i.d.~from distributions $\Sc$ and $\Tc$, respectively. For any $\Hf\subseteq \Hc$ the following holds with probability $1 - \d$:
\begin{align}
\Tt_\Hf(\Sc\|\Tc) \leq \Tt_\Hf(\widehat{\Sc}\| \widehat{\Tc})+ 4\Rf_m(\Fc_\Hf) + 4\Rf_k(\Fc_\Hf)  + 2\sqrt{\frac{\log(4/\d)}{2m}} + 2\sqrt{\frac{\log(4/\d)}{2k}}, \nonumber
\end{align}
where $\Fc_{\Hf} := \{(x, y) \mapsto \one(h(x) \neq y),\, h\in \Hf\}$. If furthermore, $\Hf$ is a set of binary classifiers with labels $\{-1, 1\}$, then $2 \Rf_m(\Fc_{\Hf}) = \Rf_m(\Hf), \, 2 \Rf_k(\Fc_{\Hf}) = \Rf_k(\Hf)$.
\end{restatable}

We also provide estimation error results using Vapnik–Chervonenkis (VC) dimension and Natarajan dimension in Appendix~\ref{app:est_vc_natarajan}. It is worth mentioning that the VC dimension of piecewise-polynomial neural networks has been upper bounded in \citet{bartlett2019nearly}. Since transfer measures can be estimated, in later sections we do not distinguish the sample sets $\widehat{\Sc}, \widehat{\Tc}$ and the underlying distributions $\Sc, \Tc$.  

\subsection{Transferability with a surrogate loss}\label{sec:surrogate}
Due to the intractability of minimizing the 0-1 loss, we need to use a surrogate loss~\citep{bartlett2006convexity} for training in practice. In this section, we discuss this nuance w.r.t.~transferability. We will focus on the most commonly used surrogate loss, cross entropy (CE), although some of the results can be easily adapted to other loss functions. To distinguish a surrogate loss from the 0-1 loss, we use $\e_\Dc$ from now on for a surrogate loss and $\e_\Dc^{\rm 0-1}$ for the 0-1 loss. One of the difficulties is the non-equivalence between $\d$-minimal sets w.r.t.~the 0-1 loss and a surrogate loss, i.e.~$\argmin(\e_\Dc, \d_\Dc)$ might be quite different from $\argmin(\e_\Dc^{\rm 0-1}, \d_\Dc)$.
Moreover, it is not practical to find all elements in $\argmin(\e_\Dc^{\rm 0-1}, \d_\Dc)$ since the loss is nonconvex and nonsmooth. In light of these difficulties, we propose a more practical notion of transferability based on surrogate loss $\e_\Dc$:

\begin{restatable}[\textbf{transfer measure with a surrogate loss}]{prop}{transferSurrogate}\label{prop:surrogate_transfer} 
Given surrogate loss $\e_\Dc \geq \e_\Dc^{0-1}$ on a general domain $\Dc$. Suppose $\Hf = \argmin(\e_\Sc, \d_\Sc)$ and denote $\e_\Tc^* = \inf_{h\in \Hf} \e_\Tc(h)$, $\e_\Sc^* = \inf_{h\in \Hf} \e_\Sc(h)$, $(\e_\Tc^{\rm 0-1})^* = \inf_{h\in \Hc} \e_\Tc^{\rm 0-1}(h)$. If the following holds:
\be\label{eq:surrogate_measure}
\Tt_\Hf^{\rm  }(\Sc\| \Tc) = \sup_{h\in \Hf} \e_\Tc (h) - \e_\Tc^* - (\e_\Sc(h) - \e_\Sc^*) \leq \d, 
\en
then we have ${\argmin} (\e^{\rm  }_\Sc, \d_\Sc) \subseteq {}{\argmin} (\e^{\rm 0-1}_\Tc,  \d + \d_\Sc + \e_\Tc^*- (\e_\Tc^{\rm 0-1})^*).$
\end{restatable}


This proposition implies that if the transfer measure is small, then a near-optimal classifier of the surrogate loss in the source domain would be near-optimal in the target domain for the 0-1 loss. It also gives us a practical framework to guarantee transferability, which we will discuss in more depth in Section \ref{sec:alg}. Assume $\e_\Dc : \Hc \to \Rb$ to be Lipschitz continuous and strongly convex, which is satisfied for the cross entropy loss (see \Cref{app:functional_surrogate}). We are able to translate the $\d$-minimal set to $L_p$ balls in the function space:
\begin{align}\label{eq:minimal_functional}
C_1 \| h - h^*\|_{2, \Dc} \leq \e_{\Dc}(h) - \e_{\Dc}(h^*) \leq C_2 \| h - h^*\|_{1, \Dc}, 
\end{align}
where $C_1$ and $C_2$ are absolute constants and $h^*$ is an optimal classifier. The function norms $\|\cdot\|_{1, \Dc}$ and $\|\cdot\|_{2, \Dc}$ are the usual $L_p$ norms over distribution $\Dc$. Since the classifier $h = q(\theta, \cdot)$ is usually parameterized with, say a neural network, we further upper bound the function norms by the distance of parameters, that is, for $1\leq p < \infty$, $h = q(\theta, \cdot)$ and $h' = q(\theta', \cdot)$, we have $\| h - h'\|_{p, \Dc} \leq L \| \theta - \theta'\|_2,$
with $L$ some Lipschitz constant of $q$ (\Cref{app:functional_surrogate}). Combined with \eqref{eq:minimal_functional}, we obtain: 
\begin{align}\label{eq:bound_parameter}
\e_{\Dc}(h) - \e_{\Dc}(h') \leq LC_2  \| \theta - \theta'\|_2.
\end{align}
In other words, if the parameters are close enough, then the losses should not differ too much. We denote $\|\cdot\|_2$ as the Euclidean norm, and for later convenience we will omit the subscript in $\|\cdot\|_2$.

\section{Algorithms for Evaluating and Improving Transferability}\label{sec:alg}

The notion of transferability is defined w.r.t.\ domains, hence by learning feature embeddings that induce certain feature distributions, one can aim to improve transferability of two given domains. In this section we design algorithms to evaluate and improve transferability by learning such transformations. To start with, let $g:\Xc\to\Zc$ be a feature embedding (a.k.a.~featurizer), where $\Zc$ is understood to be a feature space. By a joint distribution $\Dc^g$ (or $\Sc^g$, $\Tc^g$) we mean a distribution on $g(\Xc) \times \Yc$. Formally, we are dealing with push-forwards of distributions: 
\begin{align}
\Sc^g := (g, \id) \# \Sc, \, \Tc^g := (g, \id) \# \Tc,
\end{align}
where $(g, \id): (x, y) \mapsto (g(x), y)$ is a function on $\Xc \times \Yc$. $\Sc$ and $\Tc$ here are joint distributions on $\Xc \times \Yc$, and here we specify $\Xc$ to be the space of the original signal such as an image. Since $\Sc$ and $\Tc$ cannot be changed, what we are evaluating here is the feature embedding $g$. 
The key quantity is transfer measures as in \eqref{eq:surrogate_measure}:
\begin{align}\label{eq:transfer_again}
\Tt_\Hf(\Sc^g \| \Tc^g) = \sup_{h\in \Hf} \e_{\Tc^g}(h) - \e_{\Tc^g}^* - (\e_{\Sc^g}(h) - \e_{\Sc^g}^*), \quad \Hf = \argmin(\e_{\Sc^g}, \d_{\Sc^g}).
\end{align}
Although $\Hf$ is hard to compute, we can use \eqref{eq:bound_parameter} to obtain a lower bound of \eqref{eq:transfer_again}. That is, given a parametrization of the classifier $h = q(\theta, \cdot)$ and the optimal classifier $h^* = q(\theta^*, \cdot)$, we have:
\begin{align}\label{eq:lower_bound_transfer}
\small
\Tt_\Hf(\Sc^g \| \Tc^g) &\geq \sup_{\|\theta - \theta^*\| \leq \d} \e_{\Tc^g}(h) - \e_{\Sc^g}(h)  - \e_{\Tc^g}^* + \e_{\Sc^g}^* \tr
&\geq \sup_{\|\theta - {\theta^*}\| \leq \d} \e_{\Tc^g}(h) - \e_{\Sc^g}(h)  - \e_{\Tc^g}(\widehat{h^*}) \tr
&\approx \sup_{\|\theta - \widehat{\theta^*}\| \leq \d} \e_{\Tc^g}(h) - \e_{\Sc^g}(h)  - \e_{\Tc^g}(\widehat{h^*})
\end{align}
where $\d > 0$ depends on $\Hf$ and the constant in \eqref{eq:bound_parameter}. In the second and the third lines, we approximated the optimal errors $\e_{\Tc^g}^*$ and $\e_{\Sc^g}^*$ with $0 \leq \e_{\Sc^g}^* \leq \e_{\Sc^g}(\widehat{h^*}), \, 0 \leq \e_{\Tc^g}^* \leq \e_{\Tc^g}(\widehat{h^*})$, and we use the learned classifier $\widehat{h^*} = q(\widehat{\theta^*}, \cdot)$ as a surrogate for the optimal classifier. As a result, if the r.h.s.~of \eqref{eq:lower_bound_transfer} is large, then $\Sc^g$ is not quite transferable to $\Tc^g$. 

\setlength{\textfloatsep}{1pt}
\begin{algorithm}\label{alg:measure_transfer}
\textbf{Input: }learned feature embedding $g$, learned classifier $\widehat{h^*} = q(\widehat{\theta^*}, \cdot)$, target sample training set $\Tc = \Sc_0$, sample training sets  $\Sc_1$, \dots, $\Sc_n$, ascent optimizer, minimal errors $\e_{\Sc_i}^* \approx \e_{\Sc_i}(\widehat{h^*})$, adversarial radius $\d$ \\
\textbf{Initialize:} a classifier $h = q(\theta, \cdot)$ and $\theta = \widehat{\theta^*}$, gap $ = -\infty$ \\
\For{$t$ in $1 \dots T$}{
Find $\max_i \e_{\Sc_i}(h\circ g)$ and $\min_i \e_{\Sc_i}(h\circ g)$ 
and corresponding indices $j$ and $k$
\\
Run an ascent optimizer on $h$ to maximize $ {\rm gap}_0 = \e_{\Sc_j}(h\circ g) - \e_{\Sc_k}(h\circ g)$ \\
Project $\theta$ onto the Euclidean ball $\|\theta - \widehat{\theta^*}\| \leq \d$\\
\If{${\rm gap}_0 > {\rm gap}$}{${\rm gap} = {\rm gap}_0$, save accuracies and losses of each domain}
}
\textbf{Output: } $j$, $k$, $h$, $\e_{\Sc_j}(h\circ g) -\e_{\Sc_k}(h\circ g)$, $\e_{\Sc_j}(\widehat{h^*})$, $\e_{\Sc_k}(\widehat{h^*})$
\caption{Algorithm for evaluating transferability among multiple domains}
\label{alg:measure_transfer_whole}
\end{algorithm}

We can thus design an algorithm to evaluate the transferability in Section~\ref{sec:eval_transfer}. By computing the lower bound in \eqref{eq:lower_bound_transfer}, we can disprove the transferability as in Prop.~\ref{prop:transfer_equiv} and Prop.~\ref{prop:surrogate_transfer}. Computing the lower bound in \eqref{eq:lower_bound_transfer} can be regarded as an attack method: there is an adversary trying to show that $\Sc^g$ is not transferable to $\Tc^g$. For this attack, we could also design a defence method aiming to minimize the lower bound and learn more transferable features.

\vspace{-0.2em}
\subsection{Algorithm for evaluating transferability}\label{sec:eval_transfer}
\vspace{-0.2em}

In domain generalization we have one target domain and more than one source domains. To ease the presentation, we denote $\Sc_0 = \Tc$ (and thus $ \Sc_0^g = \Tc^g$) and extend the index set to be $\{0, 1, \cdots, n\}$. We need to evaluate the transferability \eqref{eq:lower_bound_transfer} between all pairs of $\Sc^g_i$ and $\Sc^g_j$. Algorithm \ref{alg:measure_transfer} gives an efficient method to compute the worst-case gap $\sup_{\|\theta - \widehat{\theta^*}\| \leq \d} \e_{\Sc^g_i}(h) - \e_{\Sc^g_j}(h)$ among all pairs of $(i, j)$. Essentially, it finds the worst pair of $(i,j)$ at each step such that 
the gap $\e_{\Sc^g_i}(h) - \e_{\Sc^g_j}(h)$ takes the largest value, and then maximize this gap over parameter $\theta$ through gradient ascent. 

Note that the computation of \eqref{eq:lower_bound_transfer} also depends on the information from the target domain. This is valid since we are only \emph{evaluating} but not \emph{training} over these domains. 


\subsection{Algorithm for improving transferability}

The evaluation sub-procedure provides us a way to pick a pair of non-transferable domains $(\Sc_i^g, \Sc_j^g)$, which in turn could be used to improve the transferability among all source domains by updating the feature embedding $g$ such that the gap $\sup_{\|\theta - \theta^*\| \leq \d} \e_{\Sc^g_i}(h) - \e_{\Sc^g_j}(h)$ for $(i, j) \in [n]\times[n]$. Simultaneously, we also require that the feature embedding $g$ preserves information for the target task of interest. With the parametrization $h = q(\theta, \cdot)$, $h' = q(\theta', \cdot)$, the overall optimization problem can be formulated as:
\begin{align}\label{eq:alg_optimization}
\min_{g, h} \max_{\|\theta' - \theta\| \leq \d} \frac{1}{n}\sum_{i=1}^n \e_{\Sc_i}^{\rm }(h \circ g) + \left( {\rm max}_i \e_{\Sc_i}^{\rm }(h'\circ g)  - {\rm min}_i \e_{\Sc_i}^{\rm }(h'\circ g) \right).
\end{align}
Intuitively, we want to learn a common feature embedding and a classifier such that all source errors are small and the pairwise transferability between source domains is also small. If the optimization problem is properly solved, then we have the following guarantee:
\begin{restatable}[\textbf{optimization guarantee}]{thm}{OptGuarantee}\label{thm:alg_framework}
Assume that the function $q(\cdot, x)$ is $L_{\theta}$ Lipschitz continuous for any $x$. Suppose we have learned a feature embedding $g$ and a classifier $h$ such that the loss functional $\e_{\Sc_i^g}: \Hc \to \Rb$ is $L_{\ell}$ Lipschitz continuous w.r.t.~distribution $\Sc_i^g$ for $i \in [n]$ and 
\be\label{eq:optimization_guarantee}
\small
\max_{\|\theta' - \theta\| \leq \d} \frac{1}{n}\sum_{i=1}^n \e_{\Sc_i}^{\rm }(h \circ g) + \left( {\rm max}_i \e_{\Sc_i}^{\rm }(h'\circ g)  - {\rm min}_i \e_{\Sc_i}^{\rm }(h'\circ g) \right) \leq \eta,
\en
where $\theta, \theta'$ are parameters of $h$ and $h'$. Then for any $h' \in \Hf = \{q(\theta', \cdot) :\|\theta - \theta'\| \leq \d\}$, we have:
\be\label{eq:alg_result_to_prove}
\small
\Tt_\Hf^{\rt}(\Tc^g_1, \Tc^g_2) \leq \eta, \quad \e_{\Sc_i}^{\rm }(h'\circ g) \leq \eta + L_{\ell}L_{\theta} \d, \quad \e_{\Tc}^{\rm }(h'\circ g) \leq 2 \eta +  L_{\ell}L_{\theta} \d,
\en
for any $\Tc_1^g, \Tc_2^g, \Tc^g \in \conv(\Sc_1^g, \dots, \Sc_n^g)$ and any $i \in [n]$. 
\end{restatable}
The Lipschitzness assumption for $\e_{\Sc_i^g}$ is mild and can be satisfied for cross entropy loss (c.f.~\Cref{app:lipschitz_cont_loss}). Here $\conv(\cdot)$ denotes the convex hull in the same sense as \citet{albuquerque2019generalizing}, i.e., each element is a mixture of source distributions.
Thm~\ref{thm:alg_framework} tells us that if we can solve the optimization problem \eqref{eq:alg_optimization} properly, we can guarantee transferability on a neighborhood of the classifier, as an approximation of the $\delta$-minimal set. We thus propose Algorithm \ref{alg:train_transfer}, which shares similarity with existing frameworks, such as DANN \citep{ganin2015unsupervised} and Distributional Robust Optimization \citep{sinha2017certifying, Sagawa*2020Distributionally}, in the sense that they all involve adversarial training and minimax optimization. However, the objective in our case is different and we provide a more detailed comparison with existing methods in Appendix~\ref{app:comparison}. 

\begin{algorithm}[H]
\textbf{Input:} samples sets of source domains $\Sc_1, \dots, \Sc_n$, feature embedding $g$, classifier $h = q(\theta, \cdot)$, adversarial classifier $h' = q(\theta', \cdot)$, surrogate loss $\e_{\Dc}$, adversarial radius $\d$, ascent optimizer, descent optimizer, weight parameter $\lambda$, number of epochs $T$ \\
\For{$t$ in $1 \dots T$}{
Compute $\max_i \e_{\Sc_i}(h\circ g)$ and $\min_i \e_{\Sc_i}(h\circ g)$ \\
Initialization $h' = h$ (or $\theta' = \theta$) \\
\For{$k$ in $1 \dots N$}{
Run the ascent optimizer on $h'$ to maximize $\max_i \e_{\Sc_i}(h'\circ g) - \min_i \e_{\Sc_i}(h'\circ g)$ fixing $g$ \\
Project $\theta'$ onto the Euclidean ball $\|\theta' - \theta\| \leq \d$}
Fixing $h'$, run the descent optimizer on $g, h$ to minimize ${\rm error} = \frac{1}{n}\sum_i \e_{\Sc_i}(h\circ g) + (\max_i \e_{\Sc_i}(h'\circ g) - \min_i \e_{\Sc_i}(h'\circ g))$
}
\textbf{Output:} feature embedding $g$, classifier $h$
\caption{Transfer algorithm for domain generalization}
\label{alg:train_transfer}
\end{algorithm}

\section{Experiments}\label{sec:exp}

\citet{gulrajani2020search} did extensive experiments on comparing DG algorithms, using the same neural architecture and data split. Specifically, they show that with data augmentation, ERM perform relatively well among a large array of algorithms. Our experiments are based on their settings. We run Algorithm \ref{alg:measure_transfer_whole} on standard benchmarks, including RotatedMNIST \citep{ghifary2015domain}, PACS \citep{li2017deeper},  Office-Home \citep{venkateswara2017deep} and WILDS-FMoW \citep{koh2020wilds} (c.f.~\Cref{app:datasets}). Specifically, WILDS-FMoW is a large dataset with nearly half a million images. Detailed experimental settings can be seen at Appendix \ref{app:add_exp}.

\setlength{\textfloatsep}{5pt}
\begin{figure}
\centering
\includegraphics[width=\textwidth]{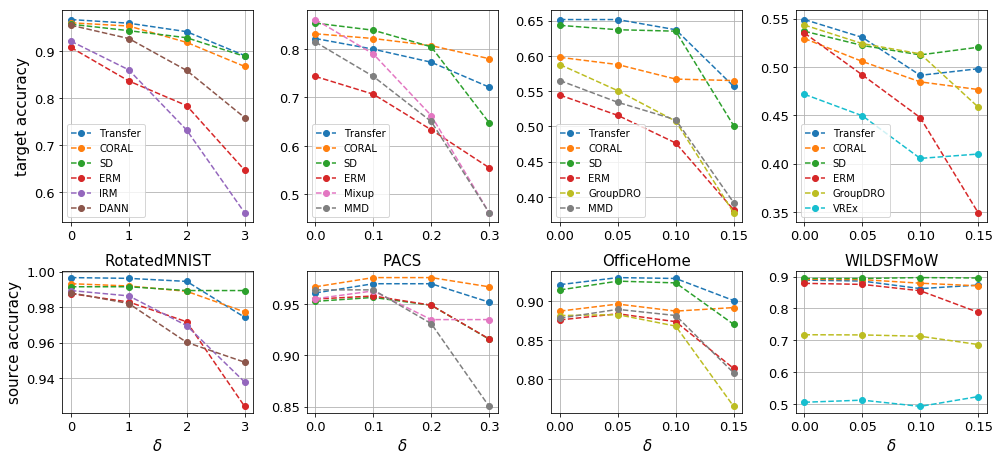}
\caption{\textbf{Top row}: test accuracy of the target domain; \textbf{bottom row}: test accuracy of one of the source domains. Each column is for a given dataset with the name in the middle, and the legends on the bottom row are the same as those on the top row. $\delta$ is the parameter in Algorithm \ref{alg:measure_transfer}. }
\label{fig:exp}
\end{figure}

\paragraph{Evaluating transferability}\;  From Figure \ref{fig:exp} it can be seen that at a neighborhood of the learned classifier, there exists a classifier such that the target accuracy is degraded significantly, whereas some source domain still has high accuracy. This poses questions to whether current popular algorithms such as ERM \citep{vapnik1992principles}, DANN \citep{ganin2016domain} and Mixup \citep{xu2020adversarial, yan2020improve} are really learning invariant and transferable features. If so, the target accuracy should be high given a high source accuracy. However, for the PACS dataset and Mixup model (the second column of Figure \ref{fig:exp}), the target accuracy decreases by more than $30\%$ while the source accuracy remains roughly at the same level. We can also, e.g., read from the first column that with a small decrease of the source (test) accuracy by $\sim 2\%$ (at $\delta=2$), the target accuracy of DANN drops by $\sim 10\%$. 

From Figure \ref{fig:exp} we can also see that Correlation Alignment \citep[CORAL,][]{sun2016deep} and Spectral Decomposition \citep[SD,][]{pezeshki2020gradient} have better transferability that other algorithms. In some sense, they are in fact learning \emph{robust classifiers}, i.e., all the classifiers on the neighborhood of the learned classifier can achieve good accuracies. With this robust classifier, the target accuracy does not decrease much even if the classifier is perturbed.

\paragraph{Improving transferability} \; \Cref{alg:train_transfer} has good performance among all four datasets that we tried, comparable to CORAL and SD. Note that CORAL and SD do not always perform well, such as in the  Office-Home and WILDS-FMoW datasets, but our Transfer algorithm does. However, in our experiments we find there are two limitations of Algorithm \ref{alg:train_transfer}: {\bf (1)} we need a large number of inner maximization steps to compute the gap, which needs more training time. This is similar to adversarial robustness \citep{madry2018towards} which is slower than usual training. In order to overcome this difficulty we used pretraining from other algorithms in the experiments on  Office-Home and WILDS-FMoW; {\bf (2)} Moderate hyper-parameter tuning is needed. For example, we need to tune $N$ is Algorithm \ref{alg:train_transfer}, the learning rate (\texttt{lr}) of SGA and the choice of $\delta$. We find that taking $N = 20$ or $30$ is usually a good choice, and $\delta$ can be quite large such that the projection step is not taken. We take $\texttt{lr}=0.01$ for RotatedMNIST and  $\texttt{lr}=0.001$ for other datasets. 

\paragraph{Label shift}\; In order to show the difference with the well-known $\Hc$-divergence \citep{ben2007analysis}, we compute the label shifts in the PACS dataset. As shown in \citet{zhao2019learning}, the optimal joint error $\e_\Sc(h^*) + \e_\Sc(h^*)$ ($h^*$ is the optimal classifier that minimizes $\e_\Sc(h) + \e_\Tc(h)$) can be large under the shift of label distributions. We follow \citep{zhao2019learning} and compute the label shift between pairs of domains in the PACS dataset, measured by total variation. From Table \ref{tab:label_shift} we can see that the label shift is large in this case, and thus the $\Hc$-divergence bound \citep{ben2007analysis} can be quite loose. 
Comparably, our transfer measure bound Prop.~\ref{thm:target_bound} is tighter (c.f.~Prop.~\ref{prop:compare_trans_measure_hdiv}) and therefore still useful in practice.
\begin{table}[ht]
\centering
\caption{Label shift between pairs of domains in the PACS dataset. {\bf TV:} total variation; {\bf A:} art painting; {\bf C}: cartoon; {\bf P}: photo; {\bf S}: sketch. The total variation is always between zero and one.}
\label{tab:label_shift}
    \begin{tabular}{ccccc}
        {\bf TV} & {\bf A} & {\bf C} & {\bf P} & {\bf S} \\
        \midrule
        {\bf A} & 0.0 & 0.12 & 0.11 & 0.3 \\
        {\bf C} & 0.12 & 0.0 & 0.18 & 0.24 \\
        {\bf P} & 0.11 & 0.18 & 0.0 & 0.37 \\
        {\bf S} & 0.3 & 0.24 & 0.37 & 0.0 \\
        \bottomrule
    \end{tabular}

\end{table}


\vspace{-0.3em}
\section{Related Work} 
\vspace{-0.3em}


\textbf{Multi-task learning}\, Multi-task learning (MTL)~\citep{zhang2017survey} is related to but different from DG. In MTL, there are several tasks, and one hopes to improve the performance of each task by jointly training all the tasks simultaneously, utilizing the relationships between them. This is different from DG in the sense that in DG the target domain is unknown a priori, whereas in MTL the focus is more on better generalization on existing tasks that appear in training. Hence, there is no distribution shift in MTL per se. Furthermore, for MTL, the output spaces of different tasks are not necessarily the same.

\textbf{Zero-shot learning / Few-shot learning / Meta-learning}\, DG is different from zero-shot learning \citep{lampert2009learning}. In zero-shot learning, one has labeled training data and the goal is to make predictions on a new unseen label set. However, in DG the label set remains the same for the source and the target domains. On the other hand, the focus of few-shot learning is on fast adaptation, in the sense that the test distribution remains the same as the training distribution, but the learner can only have access to a few labeled samples. Domain generalization also shares similarity with meta-learning. However, in meta-learning, the learner is allowed to fine-tune over the target domain. In other words, the protocol of meta-learning allows access to a small amount of labeled data from future unseen domains. Meta-learning is more or less one specific method that is used to tackle few-shot learning. Because of the similarity, some meta-learning algorithms can be applied to DG~\citep{li2018learning}. 

\textbf{Self-supervised learning} Self-supervised learning (SSL) is a popular unsupervised feature learning approach \citep{ChenKNH20, he2020momentum, Grill2020bootstrap}. The goal of SSL is to learn invariant representations w.r.t.~different views of the same image. Although it is a promising feature learning method, it differs from our DG settings in the sense that no labels are used in SSL. 

\textbf{Domain generalization}
There have been a lot of old and new algorithms proposed for domain generalization. The simplest one is Empirical Risk Minimization (ERM), where we simply minimize the empirical risk of (the sum of) all source domains. In \citet{blanchard2011generalizing, muandet2013domain}, kernel methods for DG were proposed. \citet{arjovsky2019invariant} proposed Invariant Risk Minimization (IRM) which aims to learn invariant predictors across source domains, and follow-up discussions can be found in \cite{rosenfeld2020risks, koyama2021invariance}. Another approach is called distributional robustness \citep{volpi2018generalizing, Sagawa*2020Distributionally}, where the model is optimized over a worst-case distribution under the constraint that this distribution is generated from a small perturbation around the source distributions. In \citet{albuquerque2019generalizing}, a DG scheme based on distribution matching was proposed. Moreover, many domain adaptation algorithms can be directly adapted to the task of domain generalization, such as CORAL \citep{sun2016deep} and DANN \citep{ganin2015unsupervised}. Last but not least, we mention a concurrent work \citep{ye2021theoretical} on the theory of domain generalization, which focuses more on proposing a model selection rule based on accuracy and variation.

\textbf{Adversarial robustness} Our evaluation and training methods in \S\ref{sec:alg} are reminiscent of the adversarial training method~\citep{madry2018towards} in the literature of adversarial robustness. Perturbing the classifier in our case corresponds to perturbing the input data in adversarial robustness. From this perspective, our Transfer algorithm is parallel to the adversarial training method. It would be interesting to design certified robust feature embeddings, by analogy with certified robust classifiers~\citep{cohen2019certified}.

\vspace{-0.3em}
\section{Conclusions}
\vspace{-0.3em}

In this paper we formally define the notion of \emph{transferability} that we can quantify, estimate and compute.  
Our transfer measures can be understood as a special class of IPMs. They are weaker than total variation and even very dissimilar distributions could be transferable to each other. Our definition of transferability can also be naturally used to derive a generalization bound for prediction error on the target domain. Based on our theory, we propose algorithms to evaluate and improve the transferability by learning feature representations. Experiments show that, somewhat surprisingly, many existing algorithms are not quite learning transferable features. From this perspective, our transfer measures offer a novel way to evaluate the features learned from different DG algorithms. We hope that our proposal of transferability could draw the community's attention to further investigate and better understand the fundamental quantity that allows robust models under distribution shifts.

\textbf{Broader Impact}\; Reliable domain generalization models are important for practice use. Our work points out the reliability issue of DG algorithms. It is worth mentioning that our evaluation method can only disprove the transferability and survival of our attack method should not be treated as a warranty. Misunderstanding of it could lead to potential harm in practical applications.

\newpage

\section*{Acknowledgements and Funding Transparency Statement}
We thank the anonymous reviewers for their constructive comments as well as the area chair and the senior area chair for overseeing the review process. 
Resources used in preparing this research were provided, in part, by the Province of Ontario, the Government of Canada through CIFAR, and companies sponsoring the Vector Institute.
We thank NSERC and the Canada CIFAR AI Chairs program for funding support. GZ is also supported by David R. Cheriton scholarship and research grant from Vector Institute. HZ is supported by a startup funding from the Department of Computer Science at UIUC. Finally, we thank Vector Institute for providing the GPU cluster.



\bibliographystyle{plainnat}
\bibliography{DA,DG}

\begin{thebibliography}{59}
\providecommand{\natexlab}[1]{#1}
\providecommand{\url}[1]{\texttt{#1}}
\expandafter\ifx\csname urlstyle\endcsname\relax
  \providecommand{\doi}[1]{doi: #1}\else
  \providecommand{\doi}{doi: \begingroup \urlstyle{rm}\Url}\fi

\bibitem[Ahuja et~al.(2020)Ahuja, Shanmugam, Varshney, and
  Dhurandhar]{ahuja2020invariant}
Kartik Ahuja, Karthikeyan Shanmugam, Kush Varshney, and Amit Dhurandhar.
\newblock Invariant risk minimization games.
\newblock In \emph{International Conference on Machine Learning}, pages
  145--155. PMLR, 2020.

\bibitem[Albuquerque et~al.(2019)Albuquerque, Monteiro, Darvishi, Falk, and
  Mitliagkas]{albuquerque2019generalizing}
Isabela Albuquerque, Jo{\~a}o Monteiro, Mohammad Darvishi, Tiago~H Falk, and
  Ioannis Mitliagkas.
\newblock Generalizing to unseen domains via distribution matching.
\newblock \emph{arXiv preprint arXiv:1911.00804}, 2019.

\bibitem[Arjovsky et~al.(2019)Arjovsky, Bottou, Gulrajani, and
  Lopez-Paz]{arjovsky2019invariant}
Martin Arjovsky, L{\'e}on Bottou, Ishaan Gulrajani, and David Lopez-Paz.
\newblock Invariant risk minimization.
\newblock \emph{arXiv preprint arXiv:1907.02893}, 2019.

\bibitem[Arora et~al.(2017)Arora, Ge, Liang, Ma, and
  Zhang]{arora2017generalization}
Sanjeev Arora, Rong Ge, Yingyu Liang, Tengyu Ma, and Yi~Zhang.
\newblock Generalization and equilibrium in generative adversarial nets
  ({GAN}s).
\newblock In \emph{International Conference on Machine Learning}, pages
  224--232. PMLR, 2017.

\bibitem[Bartlett and Mendelson(2002)]{bartlett2002rademacher}
Peter~L Bartlett and Shahar Mendelson.
\newblock Rademacher and gaussian complexities: Risk bounds and structural
  results.
\newblock \emph{Journal of Machine Learning Research}, 3\penalty0
  (Nov):\penalty0 463--482, 2002.

\bibitem[Bartlett et~al.(2006)Bartlett, Jordan, and
  McAuliffe]{bartlett2006convexity}
Peter~L Bartlett, Michael~I Jordan, and Jon~D McAuliffe.
\newblock Convexity, classification, and risk bounds.
\newblock \emph{Journal of the American Statistical Association}, 101\penalty0
  (473):\penalty0 138--156, 2006.

\bibitem[Bartlett et~al.(2019)Bartlett, Harvey, Liaw, and
  Mehrabian]{bartlett2019nearly}
Peter~L Bartlett, Nick Harvey, Christopher Liaw, and Abbas Mehrabian.
\newblock Nearly-tight {VC}-dimension and pseudodimension bounds for piecewise
  linear neural networks.
\newblock \emph{J. Mach. Learn. Res.}, 20\penalty0 (63):\penalty0 1--17, 2019.

\bibitem[Ben-David et~al.(2007)Ben-David, Blitzer, Crammer, and
  Pereira]{ben2007analysis}
Shai Ben-David, John Blitzer, Koby Crammer, and Fernando Pereira.
\newblock Analysis of representations for domain adaptation.
\newblock In \emph{Advances in neural information processing systems}, pages
  137--144, 2007.

\bibitem[Ben-David et~al.(2010)Ben-David, Blitzer, Crammer, Kulesza, Pereira,
  and Vaughan]{ben2010theory}
Shai Ben-David, John Blitzer, Koby Crammer, Alex Kulesza, Fernando Pereira, and
  Jennifer~Wortman Vaughan.
\newblock A theory of learning from different domains.
\newblock \emph{Machine learning}, 79\penalty0 (1-2):\penalty0 151--175, 2010.

\bibitem[Blanchard et~al.(2011)Blanchard, Lee, and
  Scott]{blanchard2011generalizing}
Gilles Blanchard, Gyemin Lee, and Clayton Scott.
\newblock Generalizing from several related classification tasks to a new
  unlabeled sample.
\newblock \emph{Advances in neural information processing systems},
  24:\penalty0 2178--2186, 2011.

\bibitem[Blanchard et~al.(2021)Blanchard, Deshmukh, Dogan, Lee, and
  Scott]{blanchard2021domain}
Gilles Blanchard, Aniket~Anand Deshmukh, Urun Dogan, Gyemin Lee, and Clayton
  Scott.
\newblock Domain generalization by marginal transfer learning.
\newblock \emph{Journal of Machine Learning Research}, 22\penalty0
  (2):\penalty0 1--55, 2021.

\bibitem[Blitzer et~al.(2007)Blitzer, Crammer, Kulesza, Pereira, and
  Wortman]{blitzer2007learning}
John Blitzer, Koby Crammer, Alex Kulesza, Fernando Pereira, and Jennifer
  Wortman.
\newblock Learning bounds for domain adaptation.
\newblock In \emph{Advances in neural information processing systems}, 2007.

\bibitem[Chen et~al.(2020)Chen, Kornblith, Norouzi, and Hinton]{ChenKNH20}
Ting Chen, Simon Kornblith, Mohammad Norouzi, and Geoffrey Hinton.
\newblock A simple framework for contrastive learning of visual
  representations.
\newblock In \emph{Proceedings of the 37th International Conference on Machine
  Learning}, pages 1597--1607, 2020.
\newblock URL \url{http://proceedings.mlr.press/v119/chen20j.html}.

\bibitem[Christie et~al.(2018)Christie, Fendley, Wilson, and
  Mukherjee]{christie2018functional}
Gordon Christie, Neil Fendley, James Wilson, and Ryan Mukherjee.
\newblock Functional map of the world.
\newblock In \emph{Proceedings of the IEEE Conference on Computer Vision and
  Pattern Recognition}, pages 6172--6180, 2018.

\bibitem[Cohen et~al.(2019)Cohen, Rosenfeld, and Kolter]{cohen2019certified}
Jeremy Cohen, Elan Rosenfeld, and Zico Kolter.
\newblock Certified adversarial robustness via randomized smoothing.
\newblock In \emph{International Conference on Machine Learning}, pages
  1310--1320. PMLR, 2019.

\bibitem[Ganin and Lempitsky(2015)]{ganin2015unsupervised}
Yaroslav Ganin and Victor Lempitsky.
\newblock Unsupervised domain adaptation by backpropagation.
\newblock In \emph{International conference on machine learning}, pages
  1180--1189. PMLR, 2015.

\bibitem[Ganin et~al.(2016)Ganin, Ustinova, Ajakan, Germain, Larochelle,
  Laviolette, Marchand, and Lempitsky]{ganin2016domain}
Yaroslav Ganin, Evgeniya Ustinova, Hana Ajakan, Pascal Germain, Hugo
  Larochelle, Fran{\c{c}}ois Laviolette, Mario Marchand, and Victor Lempitsky.
\newblock Domain-adversarial training of neural networks.
\newblock \emph{The Journal of Machine Learning Research}, 17\penalty0
  (1):\penalty0 2096--2030, 2016.

\bibitem[Ghifary et~al.(2015)Ghifary, Kleijn, Zhang, and
  Balduzzi]{ghifary2015domain}
Muhammad Ghifary, W~Bastiaan Kleijn, Mengjie Zhang, and David Balduzzi.
\newblock Domain generalization for object recognition with multi-task
  autoencoders.
\newblock In \emph{Proceedings of the IEEE international conference on computer
  vision}, pages 2551--2559, 2015.

\bibitem[Grill et~al.(2020)Grill, Strub, Altché, Tallec, Richemond,
  Buchatskaya, Doersch, Ávila Pires, Guo, Azar, Piot, Kavukcuoglu, Munos, and
  Valko]{Grill2020bootstrap}
Jean-Bastien Grill, Florian Strub, Florent Altché, Corentin Tallec, Pierre~H.
  Richemond, Elena Buchatskaya, Carl Doersch, Bernardo Ávila Pires, Zhaohan
  Guo, Mohammad~Gheshlaghi Azar, Bilal Piot, Koray Kavukcuoglu, Rémi Munos,
  and Michal Valko.
\newblock Bootstrap your own latent - a new approach to self-supervised
  learning.
\newblock In \emph{NeurIPS}, 2020.
\newblock URL
  \url{https://proceedings.neurips.cc/paper/2020/hash/f3ada80d5c4ee70142b17b8192b2958e-Abstract.html}.

\bibitem[Gulrajani and Lopez-Paz(2020)]{gulrajani2020search}
Ishaan Gulrajani and David Lopez-Paz.
\newblock In search of lost domain generalization.
\newblock \emph{arXiv preprint arXiv:2007.01434}, 2020.

\bibitem[He et~al.(2020)He, Fan, Wu, Xie, and Girshick]{he2020momentum}
Kaiming He, Haoqi Fan, Yuxin Wu, Saining Xie, and Ross Girshick.
\newblock Momentum contrast for unsupervised visual representation learning.
\newblock In \emph{Proceedings of the IEEE/CVF Conference on Computer Vision
  and Pattern Recognition}, pages 9729--9738, 2020.

\bibitem[Huang et~al.(2020)Huang, Wang, Xing, and Huang]{huang2020self}
Zeyi Huang, Haohan Wang, Eric~P. Xing, and Dong Huang.
\newblock Self-challenging improves cross-domain generalization.
\newblock In \emph{ECCV}, 2020.

\bibitem[Koh et~al.(2020)Koh, Sagawa, Marklund, Xie, Zhang, Balsubramani, Hu,
  Yasunaga, Phillips, Gao, et~al.]{koh2020wilds}
Pang~Wei Koh, Shiori Sagawa, Henrik Marklund, Sang~Michael Xie, Marvin Zhang,
  Akshay Balsubramani, Weihua Hu, Michihiro Yasunaga, Richard~Lanas Phillips,
  Irena Gao, et~al.
\newblock {WILDS}: A benchmark of in-the-wild distribution shifts.
\newblock \emph{arXiv preprint arXiv:2012.07421}, 2020.

\bibitem[Koltchinskii(2010)]{koltchinskii2010rademacher}
Vladimir Koltchinskii.
\newblock Rademacher complexities and bounding the excess risk in active
  learning.
\newblock \emph{The Journal of Machine Learning Research}, 11:\penalty0
  2457--2485, 2010.

\bibitem[Koyama and Yamaguchi(2021)]{koyama2021invariance}
Masanori Koyama and Shoichiro Yamaguchi.
\newblock When is invariance useful in an out-of-distribution generalization
  problem?, 2021.

\bibitem[Krueger et~al.(2020)Krueger, Caballero, Jacobsen, Zhang, Binas, Zhang,
  Priol, and Courville]{krueger2020out}
David Krueger, Ethan Caballero, Joern-Henrik Jacobsen, Amy Zhang, Jonathan
  Binas, Dinghuai Zhang, Remi~Le Priol, and Aaron Courville.
\newblock Out-of-distribution generalization via risk extrapolation ({Rex}).
\newblock \emph{arXiv preprint arXiv:2003.00688}, 2020.

\bibitem[Lampert et~al.(2009)Lampert, Nickisch, and
  Harmeling]{lampert2009learning}
Christoph~H Lampert, Hannes Nickisch, and Stefan Harmeling.
\newblock Learning to detect unseen object classes by between-class attribute
  transfer.
\newblock In \emph{2009 IEEE Conference on Computer Vision and Pattern
  Recognition}, pages 951--958. IEEE, 2009.

\bibitem[Li et~al.(2017)Li, Yang, Song, and Hospedales]{li2017deeper}
Da~Li, Yongxin Yang, Yi-Zhe Song, and Timothy~M Hospedales.
\newblock Deeper, broader and artier domain generalization.
\newblock In \emph{Proceedings of the IEEE international conference on computer
  vision}, pages 5542--5550, 2017.

\bibitem[Li et~al.(2018{\natexlab{a}})Li, Yang, Song, and
  Hospedales]{li2018learning}
Da~Li, Yongxin Yang, Yi-Zhe Song, and Timothy Hospedales.
\newblock Learning to generalize: Meta-learning for domain generalization.
\newblock In \emph{Proceedings of the AAAI Conference on Artificial
  Intelligence}, volume~32, 2018{\natexlab{a}}.

\bibitem[Li et~al.(2018{\natexlab{b}})Li, Pan, Wang, and Kot]{li2018domain2}
Haoliang Li, Sinno~Jialin Pan, Shiqi Wang, and Alex~C Kot.
\newblock Domain generalization with adversarial feature learning.
\newblock In \emph{Proceedings of the IEEE Conference on Computer Vision and
  Pattern Recognition}, pages 5400--5409, 2018{\natexlab{b}}.

\bibitem[Li et~al.(2018{\natexlab{c}})Li, Tian, Gong, Liu, Liu, Zhang, and
  Tao]{li2018deep}
Ya~Li, Xinmei Tian, Mingming Gong, Yajing Liu, Tongliang Liu, Kun Zhang, and
  Dacheng Tao.
\newblock Deep domain generalization via conditional invariant adversarial
  networks.
\newblock In \emph{Proceedings of the European Conference on Computer Vision
  (ECCV)}, pages 624--639, 2018{\natexlab{c}}.

\bibitem[Long et~al.(2017)Long, Zhu, Wang, and Jordan]{long2017deep}
Mingsheng Long, Han Zhu, Jianmin Wang, and Michael~I Jordan.
\newblock Deep transfer learning with joint adaptation networks.
\newblock In \emph{International conference on machine learning}, pages
  2208--2217. PMLR, 2017.

\bibitem[Madry et~al.(2018)Madry, Makelov, Schmidt, Tsipras, and
  Vladu]{madry2018towards}
Aleksander Madry, Aleksandar Makelov, Ludwig Schmidt, Dimitris Tsipras, and
  Adrian Vladu.
\newblock Towards deep learning models resistant to adversarial attacks.
\newblock In \emph{International Conference on Learning Representations}, 2018.

\bibitem[Mohri et~al.(2018)Mohri, Rostamizadeh, and
  Talwalkar]{mohri2018foundations}
Mehryar Mohri, Afshin Rostamizadeh, and Ameet Talwalkar.
\newblock \emph{Foundations of machine learning}.
\newblock MIT press, 2018.

\bibitem[Muandet et~al.(2013)Muandet, Balduzzi, and
  Sch{\"o}lkopf]{muandet2013domain}
Krikamol Muandet, David Balduzzi, and Bernhard Sch{\"o}lkopf.
\newblock Domain generalization via invariant feature representation.
\newblock In \emph{International Conference on Machine Learning}, pages 10--18.
  PMLR, 2013.

\bibitem[M{\"u}ller(1997)]{muller1997integral}
Alfred M{\"u}ller.
\newblock Integral probability metrics and their generating classes of
  functions.
\newblock \emph{Advances in Applied Probability}, pages 429--443, 1997.

\bibitem[Nam et~al.(2019)Nam, Lee, Park, Yoon, and Yoo]{nam2019reducing}
Hyeonseob Nam, HyunJae Lee, Jongchan Park, Wonjun Yoon, and Donggeun Yoo.
\newblock Reducing domain gap via style-agnostic networks.
\newblock \emph{arXiv preprint arXiv:1910.11645}, 2019.

\bibitem[Natarajan(1989)]{natarajan1989learning}
Balas~K Natarajan.
\newblock On learning sets and functions.
\newblock \emph{Machine Learning}, 4\penalty0 (1):\penalty0 67--97, 1989.

\bibitem[Peters et~al.(2016)Peters, B{\"u}hlmann, and
  Meinshausen]{peters2016causal}
Jonas Peters, Peter B{\"u}hlmann, and Nicolai Meinshausen.
\newblock Causal inference by using invariant prediction: identification and
  confidence intervals.
\newblock \emph{Journal of the Royal Statistical Society. Series B (Statistical
  Methodology)}, pages 947--1012, 2016.

\bibitem[Pezeshki et~al.(2020)Pezeshki, Kaba, Bengio, Courville, Precup, and
  Lajoie]{pezeshki2020gradient}
Mohammad Pezeshki, S{\'e}kou-Oumar Kaba, Yoshua Bengio, Aaron Courville, Doina
  Precup, and Guillaume Lajoie.
\newblock Gradient starvation: A learning proclivity in neural networks.
\newblock \emph{arXiv preprint arXiv:2011.09468}, 2020.

\bibitem[Rosenfeld et~al.(2021)Rosenfeld, Ravikumar, and
  Risteski]{rosenfeld2020risks}
Elan Rosenfeld, Pradeep~Kumar Ravikumar, and Andrej Risteski.
\newblock The risks of invariant risk minimization.
\newblock In \emph{International Conference on Learning Representations}, 2021.
\newblock URL \url{https://openreview.net/forum?id=BbNIbVPJ-42}.

\bibitem[Rudin(1987)]{rudin1987real}
Walter Rudin.
\newblock \emph{Real and complex analysis}.
\newblock McGraw-Hill Education, 1987.

\bibitem[Sagawa* et~al.(2020)Sagawa*, Koh*, Hashimoto, and
  Liang]{Sagawa*2020Distributionally}
Shiori Sagawa*, Pang~Wei Koh*, Tatsunori~B. Hashimoto, and Percy Liang.
\newblock Distributionally robust neural networks.
\newblock In \emph{International Conference on Learning Representations}, 2020.
\newblock URL \url{https://openreview.net/forum?id=ryxGuJrFvS}.

\bibitem[Shalev-Shwartz and Ben-David(2014)]{shalev2014understanding}
Shai Shalev-Shwartz and Shai Ben-David.
\newblock \emph{Understanding machine learning: From theory to algorithms}.
\newblock Cambridge university press, 2014.

\bibitem[Sinha et~al.(2017)Sinha, Namkoong, Volpi, and
  Duchi]{sinha2017certifying}
Aman Sinha, Hongseok Namkoong, Riccardo Volpi, and John Duchi.
\newblock Certifying some distributional robustness with principled adversarial
  training.
\newblock \emph{arXiv preprint arXiv:1710.10571}, 2017.

\bibitem[Sriperumbudur et~al.(2012)Sriperumbudur, Fukumizu, Gretton,
  Sch{\"o}lkopf, Lanckriet, et~al.]{sriperumbudur2009integral}
Bharath~K Sriperumbudur, Kenji Fukumizu, Arthur Gretton, Bernhard
  Sch{\"o}lkopf, Gert~RG Lanckriet, et~al.
\newblock On the empirical estimation of integral probability metrics.
\newblock \emph{Electronic Journal of Statistics}, 6:\penalty0 1550--1599,
  2012.

\bibitem[Sun and Saenko(2016)]{sun2016deep}
Baochen Sun and Kate Saenko.
\newblock Deep {CORAL}: Correlation alignment for deep domain adaptation.
\newblock In \emph{European conference on computer vision}, pages 443--450.
  Springer, 2016.

\bibitem[Tachet~des Combes et~al.(2020)Tachet~des Combes, Zhao, Wang, and
  Gordon]{tachet2020domain}
Remi Tachet~des Combes, Han Zhao, Yu-Xiang Wang, and Geoffrey~J Gordon.
\newblock Domain adaptation with conditional distribution matching and
  generalized label shift.
\newblock \emph{Advances in Neural Information Processing Systems}, 33, 2020.

\bibitem[Tzeng et~al.(2017)Tzeng, Hoffman, Saenko, and
  Darrell]{tzeng2017adversarial}
Eric Tzeng, Judy Hoffman, Kate Saenko, and Trevor Darrell.
\newblock Adversarial discriminative domain adaptation.
\newblock In \emph{Proceedings of the IEEE conference on computer vision and
  pattern recognition}, pages 7167--7176, 2017.

\bibitem[Vapnik(1992)]{vapnik1992principles}
Vladimir Vapnik.
\newblock Principles of risk minimization for learning theory.
\newblock In \emph{Advances in neural information processing systems}, pages
  831--838, 1992.

\bibitem[Venkateswara et~al.(2017)Venkateswara, Eusebio, Chakraborty, and
  Panchanathan]{venkateswara2017deep}
Hemanth Venkateswara, Jose Eusebio, Shayok Chakraborty, and Sethuraman
  Panchanathan.
\newblock Deep hashing network for unsupervised domain adaptation.
\newblock In \emph{({IEEE}) Conference on Computer Vision and Pattern
  Recognition ({CVPR})}, 2017.

\bibitem[Volpi et~al.(2018)Volpi, Namkoong, Sener, Duchi, Murino, and
  Savarese]{volpi2018generalizing}
Riccardo Volpi, Hongseok Namkoong, Ozan Sener, John~C Duchi, Vittorio Murino,
  and Silvio Savarese.
\newblock Generalizing to unseen domains via adversarial data augmentation.
\newblock In \emph{NeurIPS}, 2018.

\bibitem[Xu et~al.(2020)Xu, Zhang, Ni, Li, Wang, Tian, and
  Zhang]{xu2020adversarial}
Minghao Xu, Jian Zhang, Bingbing Ni, Teng Li, Chengjie Wang, Qi~Tian, and
  Wenjun Zhang.
\newblock Adversarial domain adaptation with domain mixup.
\newblock In \emph{Proceedings of the AAAI Conference on Artificial
  Intelligence}, volume~34, pages 6502--6509, 2020.

\bibitem[Yan et~al.(2020)Yan, Song, Li, Zou, and Ren]{yan2020improve}
Shen Yan, Huan Song, Nanxiang Li, Lincan Zou, and Liu Ren.
\newblock Improve unsupervised domain adaptation with mixup training.
\newblock \emph{arXiv preprint arXiv:2001.00677}, 2020.

\bibitem[Ye et~al.(2021)Ye, Xie, Cai, Li, Li, and Wang]{ye2021theoretical}
Haotian Ye, Chuanlong Xie, Tianle Cai, Ruichen Li, Zhenguo Li, and Liwei Wang.
\newblock Towards a theoretical framework of out-of-distribution
  generalization, 2021.

\bibitem[Zhang and Yang(2021)]{zhang2017survey}
Yu~Zhang and Qiang Yang.
\newblock A survey on multi-task learning.
\newblock \emph{IEEE Transactions on Knowledge and Data Engineering}, pages
  1--1, 2021.
\newblock \doi{10.1109/TKDE.2021.3070203}.

\bibitem[Zhang et~al.(2019)Zhang, Liu, Long, and Jordan]{zhang2019bridging}
Yuchen Zhang, Tianle Liu, Mingsheng Long, and Michael Jordan.
\newblock Bridging theory and algorithm for domain adaptation.
\newblock In \emph{International Conference on Machine Learning}, pages
  7404--7413. PMLR, 2019.

\bibitem[Zhao et~al.(2018)Zhao, Zhang, Wu, Moura, Costeira, and
  Gordon]{zhao2018adversarial}
Han Zhao, Shanghang Zhang, Guanhang Wu, Jos{\'e}~MF Moura, Joao~P Costeira, and
  Geoffrey~J Gordon.
\newblock Adversarial multiple source domain adaptation.
\newblock \emph{Advances in neural information processing systems},
  31:\penalty0 8559--8570, 2018.

\bibitem[Zhao et~al.(2019)Zhao, Des~Combes, Zhang, and
  Gordon]{zhao2019learning}
Han Zhao, Remi~Tachet Des~Combes, Kun Zhang, and Geoffrey Gordon.
\newblock On learning invariant representations for domain adaptation.
\newblock In \emph{International Conference on Machine Learning}, pages
  7523--7532. PMLR, 2019.

\end{thebibliography}

\numberwithin{equation}{section}

\appendix

\newpage

\section{Proofs}\label{app:proofs}

In this appendix, we present proofs of our theoretical results in the main paper. 

\setcounter{thm}{3}

\UpperRealizable*

\begin{proof}

We first note that:
\begin{align}
\Tt_{\Hf}(\Sc, \Tc) &= \sup_{h\in \Hf}|(\e_\Sc(h) - \e_\Tc(h) - (\e_\Sc^* - \e_\Tc^*))| \leq  \sup_{h\in \Hf}|\e_\Sc(h) - \e_\Tc(h)| +  |\e_\Sc^* - \e_\Tc^*| \tr
&= \Tt_{\Hf}^\rt(\Sc, \Tc) + |\e_\Sc^* - \e_\Tc^*|.
\end{align}
It suffices to prove that $|\e_\Sc^* - \e_\Tc^*| \leq \Tt^{\mathtt{r}}_\Hf(\Sc, \Tc)$. Suppose $\Tt^{\mathtt{r}}_\Hf(\Sc, \Tc) \leq \d$, and $h_\Sc^* \in \argmin_{h\in \Hf} \e_\Sc$. Then $\e_\Sc(h_\Sc^*) = \inf_{h\in \Hf} \e_\Sc = \e_\Sc^*$ and we have:
\be
\e_\Tc^* - \d \leq \e_\Tc(h_\Sc^*) - \d \leq \e_\Sc(h_\Sc^*) = \e_\Sc^*,
\en
where in the first inequality, we used $\e_\Tc^* = \inf_{h\in \Hf} \e_\Tc(h)$, and in the second inequality, we used that for any $h\in \Hf$, we have:
\be
\e_\Tc(h) - \e_\Sc(h) \leq \sup_{h\in \Hf} |\e_\Tc(h) - \e_\Sc(h)| = \Tt^{\mathtt{r}}_\Hf(\Sc, \Tc) \leq \d.
\en
Hence, we have $\e_\Tc^* - \e_\Sc^* \leq \Tt^{\mathtt{r}}_\Hf(\Sc, \Tc)$. Since $\Tt^{\mathtt{r}}_\Hf(\Sc, \Tc)$ is symmetric in $\Tc$ and $\Sc$, we also have $\e_\Sc^* - \e_\Tc^* \leq \Tt^{\mathtt{r}}_\Hf(\Sc, \Tc)$. Hence we have proved $|\e_\Sc^* - \e_\Tc^*| \leq \Tt^{\mathtt{r}}_\Hf(\Sc, \Tc)$. 
\end{proof}

\OneSide*

\begin{proof}
If $\mathtt{T}_{\Hf}(\Sc\|\Tc) = \sup_{h\in \Hf}\e_\Tc(h) - \e_\Tc^* - (\e_\Sc(h) - \e_\Sc^*)  \leq \d$, then for any $h \in \Hf = \argmin(\e_\Sc, \d_\Sc)$ we have $\e_\Sc(h) \leq \e_\Sc^* + \d_\Sc$ and
\be
\e_\Tc(h) - \e_\Tc^* \leq \e_\Sc(h) - \e_\Sc^* + \d \leq \d_\Sc + \d.
\en
Since we assume $\e_\Tc^* = \inf_{h\in \Hf} \e_\Tc(h) = \inf_{h\in \Hc} \e_\Tc(h)$, we obtain that $\Sc$ is $(\d_\Sc, \d_\Sc + \d)$-transferable to $\Tc$.  If $\Tt_\Hf(\Sc, \Tc) \leq \d$, then $\Tt_\Hf(\Sc \| \Tc) \leq \d$, and $\Sc$ is $(\d_\Sc, \d_\Sc + \d)$-transferable to $\Tc$.

If $\Sc$ is $(\d_\Sc, \d_\Tc)$-transferable to $\Tc$, then for any $h\in {}{\argmin}(\e_\Sc, \d_\Sc)$, we have:
\be\label{eq:bound_one_side_transfer}
\e_\Tc(h) - \e_\Tc^* - (\e_\Sc(h) - \e_\Sc^*) \leq \e_\Tc(h) - \e_\Tc^* \leq \d_\Tc.
\en
We also have $\Tt_\Hf(\Sc \| \Tc) \leq \d_\Tc$ from \eqref{eq:bound_one_side_transfer}. Moreover, we can derive that:
\be
\Tt_\Hf(\Tc \| \Sc) = \sup_{h\in \Hf}  (\e_\Sc(h) - \e_\Sc^*) -(\e_\Tc(h) - \e_\Tc^*) \leq \sup_{h\in \Hf}  (\e_\Sc(h) - \e_\Sc^*) \leq \d_\Sc,
\en
and thus $\Tt_\Hf(\Sc, \Tc) \leq \max\{\d_\Sc, \d_\Tc\}$ from the definition. 
\end{proof}

\setcounter{thm}{5}

\pmetric*

\begin{proof}
$\Tt_\Hf(\Sc, \Sc) = 0$ and $\Tt_\Hf(\Sc, \Tc) = \Tt_\Hf(\Tc, \Sc)$ follow from the definition. Denote the excess risk ${\rm exc}_\Sc(h) = \e_\Sc(h) - \e_\Sc^*$ (we could change the letter $\Sc$ here). The triangle inequality can be derived as: 
\begin{align}\label{eq:composable_one_side}
\Tt_\Hf(\Sc, \Tc) &= \sup_{h\in \Hf} |{\rm exc}_\Sc(h) - {\rm exc}_\Tc(h) |\tr 
&= \sup_{h\in \Hf} |{\rm exc}_\Sc(h) - {\rm exc}_\Pc(h) + {\rm exc}_\Pc(h) - {\rm exc}_\Tc(h) | \tr
&\leq\sup_{h\in \Hf} |{\rm exc}_\Sc(h) - {\rm exc}_\Pc(h)| +\sup_{h\in \Hf}  | {\rm exc}_\Pc(h) - {\rm exc}_\Tc(h) | \tr
&= \Tt_\Hf(\Sc, \Pc) + \Tt_\Hf(\Pc, \Tc).
\end{align}
Similarly, we can derive $\Tt_\Hf(\Sc \| \Tc) \leq \Tt_\Hf(\Sc \| \Pc) + \Tt_\Hf(\Pc \| \Tc)$.
\end{proof}

\upperTV*

\begin{proof}
Let us first recall the definition of IPMs: 
\be
d_{\Fc}(\Sc, \Tc) = \sup_{f\in \Fc} \left|\sum_y \int f(x, y) (p_\Sc(x, y) - p_\Tc(x, y)) dx  \right|.
\en
The symmetric transfer measure $\TwoTransferReal$ and the total variation can be represented as:
\be
\TwoTransferReal = d_{\Fc_\Hf}(\Sc, \Tc), \, \TV = d_{\Fc_{\rm TV}}(\Sc, \Tc),
\en
with $\Fc_{\rm \Hf} := \{(x, y)\mapsto \one(h(x)\neq y), h\in \Hf\}$, $\Fc_{\rm TV} = \{f : \|f\|_{\infty}\leq 1\}$ (see also \Cref{app:other_ipm}). 
The first sentence follows from $\Fc_{\Hf} \subseteq \Fc_{\rm TV}$, and the definition of IPM.

Now let us prove the case when $\Hf = \Hc_t$ is unconstrained.
Suppose $\Tt_{\Hc_t}^\rt(\Sc, \Tc) \leq \d$, then for any binary classifier $h$, we have $|\e_\Sc(h) - \e_\Tc(h)| \leq \d$. For simplicity, denote the difference of the two distributions as: 
\be
d(x, y) := p_\Sc(x, y) - p_\Tc(x, y).
\en
\noindent
Take $h_+$ to be the following (note that we allow the classifier to take a garbage value $0$): 
\be
h_+(x) = \begin{cases}
0 & \textrm{ if }x \in \Bc_{>>}:= \{x \in \Xc: d(x, 1) \geq 0\textrm{ and }d(x, -1) \geq 0\} \\
-1 & \textrm{ if }x \in \Bc_{><}: = \{x\in \Xc: d(x, 1) \geq 0, \, d(x, -1) < 0\} \\
1 & \textrm{ if }x\in \Bc_{<>}:= \{x\in \Xc: d(x, 1) < 0,\, d(x, -1) \geq 0\} \\
1 & \textrm{ if }x\in \Bc_{<<}^- := \{x\in \Xc: d(x, 1) < d(x, -1) < 0\} \\
-1 & \textrm{ if }x\in \Bc_{<<}^+ := \{x\in \Xc: 0 > d(x, 1) \geq d(x, -1)\}
\end{cases},
\en
and denote $\Bc_{<<}:= \Bc_{<<}^- \cup \Bc_{<<}^+$. Then we have from the definition:
\begin{align}\label{eq:error_diff_ST}
\e_\Sc(h_+) - \e_\Tc(h_+) &= \sum_y \int (p_\Sc(x, y) - p_\Tc(x, y)) \one(h_+(x)\neq y) dx \tr
&= \int d(x,1) \one(h_+(x)\neq 1) + d(x, -1)  \one(h_+(x)\neq -1) dx \tr
&= \int_{\Bc_{>>}}d(x, 1) + d(x, -1) dx + \int_{\Bc_{><}}d(x, 1) dx + \int_{\Bc_{<>}}d(x, -1) dx \tr
&- \int_{\Bc_{<<}} \min\{-d(x, 1), -d(x, -1)\}dx.
\end{align}
Moreover, one can verify that $\e_\Sc(h_+) -\e_\Tc(h_+) = \sup_{h\in \Hc_t} \e_\Sc(h) -\e_\Tc(h)$. Similarly, let us define $h_-$ to be:
\be
h_-(x) = \begin{cases}
0 & \textrm{ if }x \in \Bc_{<<}:= \{x \in \Xc: d(x, 1) < 0\textrm{ and }d(x, -1) < 0\} \\
-1 & \textrm{ if }x \in \Bc_{<>}: = \{x\in \Xc: d(x, 1) < 0, \, d(x, -1) \geq 0\} \\
1 & \textrm{ if }x\in \Bc_{><}:= \{x\in \Xc: d(x, 1) \geq 0,\, d(x, -1) < 0\} \\
-1 & \textrm{ if }x\in \Bc_{>>}^- := \{x\in \Xc: 0 \leq d(x, 1) < d(x, -1)\} \\
1 & \textrm{ if }x\in \Bc_{>>}^+ := \{x\in \Xc: d(x, 1) \geq d(x, -1) \geq 0\}
\end{cases}.
\en
Then we have from the definition:
\begin{align}\label{eq:error_diff_TS}
\e_\Tc(h_-) - \e_\Sc(h_-) &= -\sum_y \int (p_\Sc(x, y) - p_\Tc(x, y)) \one(h_-(x)\neq y) dx \tr
&= \int -d(x,1) \one(h_-(x)\neq 1) - d(x, -1)  \one(h_-(x)\neq -1) dx \tr
&= \int_{\Bc_{<<}} -d(x, 1) - d(x, -1) dx + \int_{\Bc_{><}}-d(x, -1) dx + \int_{\Bc_{<>}}-d(x, 1) dx \tr
&- \int_{\Bc_{>>}} \min\{d(x, 1), d(x, -1)\}dx.
\end{align}
Moreover, $\e_\Tc(h_-) -\e_\Sc(h_-) = \sup_{h\in \Hc_t} \e_\Tc(h) -\e_\Sc(h)$. Summing over \eqref{eq:error_diff_ST} and \eqref{eq:error_diff_TS} we have:
\begin{align}\label{eq:lower_bound_es_et}
2\sup_{h\in \Hc_t} |\e_\Sc(h) - \e_\Tc(h)| &\geq |\e_\Sc(h_+) - \e_\Tc(h_+)| + |\e_\Tc(h_-) - \e_\Sc(h_-)| \tr
&\geq \e_\Sc(h_+) - \e_\Tc(h_+) + \e_\Tc(h_-) - \e_\Sc(h_-) \tr
&= \int_{\Bc_{>>}} \max\{d(x, 1), d(x, -1)\}dx + \int_{\Bc_{><}}d(x, 1) - d(x, -1) dx \tr
&+\int_{\Bc_{<>}}-d(x, 1) + d(x, -1) dx + \int_{\Bc_{<<}} \max\{-d(x, 1), -d(x, -1)\}dx.
\end{align}
On the other hand, we can compute the total variation between $\Sc$ and $\Tc$: 
\begin{align}\label{eq:dTV_ST}
d_{\rm TV}(\Sc, \Tc) &=  \sum_y \int |p_\Sc(x, y) - p_\Tc(x, y)| dx \tr
&= \int |d(x, 1)| + |d(x, -1)| dx \tr
&= \int_{\Bc_{>>}} d(x, 1) + d(x, -1) dx + \int_{\Bc_{><}}d(x, 1) - d(x, -1) dx \tr
&+\int_{\Bc_{<>}}-d(x, 1) + d(x, -1) dx + \int_{\Bc_{<<}} - d(x, 1) - d(x, -1)dx  \tr
&\leq 2\int_{\Bc_{>>}} \max\{ d(x, 1), d(x, -1) \} dx + 2\int_{\Bc_{><}}d(x, 1) - d(x, -1) dx \tr
&+2\int_{\Bc_{<>}}-d(x, 1) + d(x, -1) dx + \int_{\Bc_{<<}} 2\max\{- d(x, 1) ,- d(x, -1)\} dx \tr
&\leq 4\sup_{h\in \Hc_t} |\e_\Sc(h) - \e_\Tc(h)| = 4 \Tt_{\Hc_t}^\rt(\Sc, \Tc),
\end{align}
where in the last line we used \eqref{eq:lower_bound_es_et}. 
\end{proof}
In the proof above, we assumed a classifier $h\in \Hf$ is allowed to take a garbage value $0$ if it is not sure which label to choose. This is a mild assumption that can hold in practice. 

\renewcommand{\thethm}{\arabic{thm}'}
\setcounter{thm}{8}

\begin{restatable}[\textbf{reduction of estimation error}]{lem}{ReductionEstApp}\label{app_lem:transfer_to_normal_error}
Suppose $\widehat{\Sc}$ and $\widehat{\Tc}$ are i.i.d.~sample distributions drawn from distributions of $\Sc$ and $\Tc$, then for any $\Hf\subseteq \Hc$ we have:
\begin{align}
& \label{eq:estimation_one_side_transfer_measure_app}\Tt_\Hf(\Sc\|\Tc) \leq \Tt_\Hf(\widehat{\Sc}\| \widehat{\Tc})+ 2{\rm est}_{\Hf}(\Sc) + 2{\rm est}_{\Hf}(\Tc), \\
& \Tt_\Hf(\Sc, \Tc) \leq \Tt_\Hf(\widehat{\Sc}, \widehat{\Tc}) + 2{\rm est}_{\Hf}(\Sc) + 2{\rm est}_{\Hf}(\Tc) , \\
& \Tt^{\mathtt{r}}_\Hf(\Sc, \Tc) \leq \Tt^{\mathtt{r}}_\Hf(\widehat{\Sc}, \widehat{\Tc}) + {\rm est}_{\Hf}(\Sc) + {\rm est}_{\Hf}(\Tc),
\end{align}
where we define 
\be\label{eq:estimate_domains}
{\rm est}_{\Hf}(\Sc) = \sup_{h\in \Hf} |\e_\Sc(h) - \e_{\widehat{\Sc}}(h)|,\, {\rm est}_{\Hf}(\Tc) = \sup_{h\in \Hf} |\e_\Tc(h) - \e_{\widehat{\Tc}}(h)|.
\en
\end{restatable}
\begin{proof}
We prove the first inequality for example and others follow similarly. Note that:
\begin{align}
\e_\Tc(h) - \e_{\Tc}^* - \e_\Sc(h) + \e_\Sc^* &= \e_\Tc(h) - \e_{\widehat{\Tc}}(h) + \e_{\widehat{\Tc}}(h) - \e_{\Tc}^* - \e_{\widehat{\Tc}}^*  + \e_{\widehat{\Tc}}^*  - \e_\Sc(h)- \e_{\widehat{\Sc}}(h) + \e_{\widehat{\Sc}}(h) + \tr
&+ \e_\Sc^* - \e_{\widehat{\Sc}}^*  + \e_{\widehat{\Sc}}^* \tr
&= (\e_{\widehat{\Tc}}(h) - \e_{\hTc}^* - \e_{\hSc}(h) + \e_{\hSc}^*) + (\e_\Tc(h) - \e_{\hTc}(h)) +  (\e_{\widehat{\Tc}}^* - \e_{\Tc}^*) + 
\tr
&+ (\e_{\widehat{\Sc}}(h) - \e_\Sc(h)) + (\e_\Sc^* - \e_{\widehat{\Sc}}^*). 
\end{align}
Taking the supremum on both sides we have:
\begin{align}
\Tt_\Hf(\Sc\|\Tc) \leq \Tt_\Hf(\widehat{\Sc}\| \widehat{\Tc}) + \sup_{h\in \Hf} |\e_\Tc(h) - \e_{\widehat{\Tc}}(h)| + \e_{\widehat{\Tc}}^* -  \e_{{\Tc}}^* + \sup_{h\in \Hf}| \e_{\widehat{\Sc}}(h) - \e_{{\Sc}}(h)| + \e_\Sc^* - \e_{\hat{\Sc}}^*.
\end{align}
Take $h_{\Tc}^* \in \argmin_{h\in \Hf} \e_\Tc(h)$ to be an optimal classifier. We can derive:
\be
\e_{\hTc}^* \leq \e_{\hTc}(h_\Tc^*) \leq \e_{\Tc}^*(h_\Tc^*) + {\rm est}_{\Hf}(\Tc) = \e_{\Tc}^* +  {\rm est}_{\Hf}(\Tc).
\en
Therefore, $\e_{\hTc}^* - \e_{\Tc}^* \leq {\rm est}_{\Hf}(\Tc)$. Similarly, $\e_{\Sc}^* - \e_{\hSc}^* \leq  {\rm est}_{\Hf}(\Sc)$. Combining all those above we obtain \eqref{eq:estimation_one_side_transfer_measure_app}. 
\end{proof}

\renewcommand{\thethm}{\arabic{thm}'}

\begin{restatable}[\textbf{estimation error with Rademacher complexity}]{thm}{RadeEstApp}\label{thm:estimate_transfer_measures_app}
Given 0-1 loss $\e_\Dc = \e_\Dc^{\rm 0-1}$, suppose $\widehat{\Sc}$ and $\widehat{\Tc}$ are sample sets with $m$ and $k$ samples drawn i.i.d.~from distributions $\Sc$ and $\Tc$, respectively. For any $\Hf\subseteq \Hc$ any of the following holds w.p.~$1 - \d$:
\begin{align}
& \label{eq:estimation_one_side_transfer_measure_Rade}\Tt_\Hf(\Sc\|\Tc) \leq \Tt_\Hf(\widehat{\Sc}\| \widehat{\Tc})+ 4\Rf_m(\Fc_\Hf) + 4\Rf_k(\Fc_\Hf)  + 2\sqrt{\frac{\log(4/\d)}{2m}} + 2\sqrt{\frac{\log(4/\d)}{2k}}, \\
& \Tt_\Hf(\Sc, \Tc) \leq \Tt_\Hf(\widehat{\Sc}, \widehat{\Tc}) + 4\Rf_m(\Fc_\Hf) + 4\Rf_k(\Fc_\Hf)  + 2\sqrt{\frac{\log(4/\d)}{2m}} + 2\sqrt{\frac{\log(4/\d)}{2k}}, \\
& \Tt^{\mathtt{r}}_\Hf(\Sc, \Tc) \leq \Tt^{\mathtt{r}}_\Hf(\widehat{\Sc}, \widehat{\Tc}) + 2\Rf_m(\Fc_\Hf) + 2\Rf_k(\Fc_\Hf) + \sqrt{\frac{\log(4/\d)}{2m}} + \sqrt{\frac{\log(4/\d)}{2k}},
\end{align}
where $\Fc_{\Hf} := \{(z, y) \mapsto \one(h(z) \neq y),\, h\in \Hf\}$. If furthermore, $\Hf$ is a set of binary classifiers with labels $\{-1, 1\}$, then $2 \Rf_m(\Fc_{\Hf}) = \Rf_m(\Hf), \, 2 \Rf_k(\Fc_{\Hf}) = \Rf_k(\Hf)$.
\end{restatable}

\begin{proof}

\renewcommand{\thethm}{\arabic{thm}}

\noin We use the following lemma, which a slight adaptation of \citet{mohri2018foundations}, Theorem 3.3: 
\begin{lem}
Let $\Fc$ be a family of functions from $\Xc\times \Yc$ to $[0, 1]$. Then for any $\d > 0$, with probability at least $1 - \d$ over the draw from a distribution $\Sc$ of an i.i.d.~samples $S$ of size $m$, $\{w_i\}_{i=1}^m$, the following holds for all $f \in \Fc$, 
\be\label{eq:diff_empirical_true_mean}
\left|\Eb[f(w)] - \frac{1}{m}\sum_{i=1}^m f(w_i)\right| \leq 2\Rf_m(\Fc) + \sqrt{\frac{\log(2/\d)}{2m}}.
\en
\end{lem}
\begin{proof}
From \citet{mohri2018foundations}, Theorem 3.3, we know with probability at least $1 - \d/2$, the following holds 
\be
\Eb[f(w)] - \frac{1}{m}\sum_{i=1}^m f(w_i) \leq  2\Rf_m(\Fc) + \sqrt{\frac{\log(2/\d)}{2m}}.
\en
This result relies on applying McDiarmid's inequality on $\Phi(S) = \sup_{f\in \Fc} \Eb[f] - \frac{1}{m}\sum_{i=1}^m f(w_i)$. By repeating the same proof and applying McDiarmid's inequality on $\Phi'(S) = \sup_{f\in \Fc} \frac{1}{m}\sum_{i=1}^m f(w_i) - \Eb[f]$, we conclude that 
with probability at least $1 - \d/2$, the following holds 
\be
\frac{1}{m}\sum_{i=1}^m f(w_i) - \Eb[f(w)] \leq 2 \Rf_m(\Fc) + \sqrt{\frac{\log(2/\d)}{2m}}.
\en
Therefore, with union bound we obtain that with probability (w.p.) at least $1 - \d$, we have \eqref{eq:diff_empirical_true_mean}. 

\end{proof}

\noin Let us now go back to the proof of \Cref{thm:estimate_transfer_measures_app}. Taking $\Fc_{\Hf} = \{(z, y) \mapsto \one(h(z) \neq y),\, h\in \Hf\}$, we can derive from the theorem above that w.p.~at least $1-\d$:
\begin{align}\label{eq:upper_bound_est_Rademacher}
{\rm est}_{\Hf}(\Sc) = \sup_{h\in \Hf}|\e_\Sc(h) -  \e_{\hat{\Sc}}(h)| \leq 2\Rf_m(\Fc_{\Hf}) + \sqrt{\frac{\log(2/\d)}{2m}}.
\end{align}

With \eqref{eq:upper_bound_est_Rademacher} we know that with probability at least $1-\d/2$:
\begin{align}
{\rm est}_{\Hf}(\Sc) = \sup_{h\in \Hf}|\e_\Sc(h) -  \e_{\hat{\Sc}}(h)| \leq 2 \Rf_m(\Fc_{\Hf}) + \sqrt{\frac{\log(4/\d)}{2m}},
\end{align}
and w.p.~at least $1-\d/2$:
\begin{align}
{\rm est}_{\Hf}(\Tc) = \sup_{h\in \Hf}|\e_\Tc(h) -  \e_{\hat{\Tc}}(h)| \leq 2 \Rf_k(\Fc_{\Hf}) + \sqrt{\frac{\log(4/\d)}{2k}},
\end{align}
therefore from union bound w.p.~at least $1- \d$ we have:
\begin{align}\label{eq:union_bound_rademacher}
{\rm est}_{\Hf}(\Sc) + {\rm est}_{\Hf}(\Tc)  \leq 2 \Rf_m(\Fc_{\Hf}) + 2 \Rf_k(\Fc_{\Hf})  + \sqrt{\frac{\log(4/\d)}{2m}} + \sqrt{\frac{\log(4/\d)}{2k}}.
\end{align}
Moreover, from Lemma 3.4 of \citet{mohri2018foundations} we have 
\be
2 \Rf_m(\Fc_{\Hf}) = \Rf_m(\Hf), \, 2 \Rf_k(\Fc_{\Hf}) = \Rf_k(\Hf),
\en
for binary classification. 
The rest follows from Lemma \ref{app_lem:transfer_to_normal_error}. 
\end{proof}

\setcounter{thm}{10}
\renewcommand{\thethm}{\arabic{thm}}

\transferSurrogate*

\begin{proof}
Suppose \eqref{eq:surrogate_measure} holds and thus for any $h \in  \argmin(\e_\Sc , \d_\Sc)$ we have:
\be
\e_\Tc^{\rm 0-1}(h)  \leq \e_\Tc (h) \leq (\e_\Sc (h) -  \e_\Sc^*) + \d +  \e_\Tc^* \leq \d_{\Sc} + \d +  \e_\Tc^*.
\en
The rest follows from definitions.
\end{proof}

\begin{restatable}[\textbf{domain generalization guarantee}]{prop}{DGGuarantee}\label{prop:new_domain}
Suppose we have $n$ distributions $\Sc^g_1, \dots, \Sc^g_n$ which satisfy 
\be\label{eq:max_minus_min}
\sup_{h\in \Hf}\max_i \e_{\Sc^g_i}(h) - \min_j \e_{\Sc^g_j}(h) \leq \d.
\en
Then for any two distributions $\Tc^g_1, \Tc^g_2$ in ${\rm conv}(\Sc^g_1, \dots, \Sc^g_n)$, we have $\Tt_\Hf^{\rt}(\Tc^g_1, \Tc^g_2) \leq \d$. 
\end{restatable}

\begin{proof}
For the ease of notation we omit the superscript $g$ in the proof. We treat distributions as probabilistic measures and thus for any $h\in \Hc$, $\e_\Dc(h)$ is a linear function of $\Dc$, if we treat $\Dc$ as a probability measure. It suffices to prove for a linear function $f$, we have:
\be
|f(\sum_{i}\pi_i \Sc_i) - f(\sum_{j}\pi'_j \Sc_j)| \leq \max_i f(\Sc_i) - \min_j f(\Sc_j),
\en
where $\pi_i, \pi'_j \geq 0$ and $\sum_i \pi_i = \sum_j \pi'_j = 1$. This is because 
\begin{align}\label{eq:super_max_min}
|f(\sum_{i}\pi_i \Sc_i) - f(\sum_{j}\pi'_j \Sc_j)|  &= |f(\sum_{i, j}\pi_i \pi'_j \Sc_i) - f(\sum_{i, j}\pi_i \pi'_j \Sc_j))| \tr
&= |f(\sum_{i, j}\pi_i \pi'_j (\Sc_i - \Sc_j))| \tr
&= |\sum_{i, j}\pi_i \pi'_j f(\Sc_i - \Sc_j)| \tr
&\leq \sum_{i, j}\pi_i \pi'_j |f(\Sc_i - \Sc_j)| \tr
&\leq \max_{i, j}|f(\Sc_i) - f(\Sc_j)| \tr
&= \max_i f(\Sc_i) - \min_j f(\Sc_j).
\end{align}
The second and the third lines follow from the linearity of $f$ and the fourth line follows from triangle inequality. Therefore, taking $ f: \Dc \mapsto \e_\Dc(h)$ for any $h\in \Hf$, and $\Tc_1 = \sum_{i}\pi_i \Sc_i$, $\Tc_2 = \sum_j \pi'_j \Sc_j$, we can derive from \eqref{eq:super_max_min} that:
\begin{align}
|\e_{\Tc_1}(h) -  \e_{\Tc_2}(h)| \leq \max_i \e_{\Sc_i}(h) - \min_j \e_{\Sc_j}(h),
\end{align}
for any $h\in \Hf$. Taking the supremum over $h$ on both sides we finish the proof. 
\end{proof}

\OptGuarantee*

\begin{proof}
From \eqref{eq:optimization_guarantee} we know that:
\be\label{eq:immediate_consequence}
{\rm max}_i \e_{\Sc_i}^{\rm  }(h'\circ g)  - {\rm min}_i \e_{\Sc_i}^{\rm  }(h'\circ g) \leq {\eta},
\en
for any $h' = q(\theta', \cdot)$ and $\|\theta' - \theta\|_2 \leq \d$. Taking $h' = h$, we obtain that:
\begin{align}
\max_i \e_{\Sc_i}(h\circ g) &= \min_i \e_{\Sc_i}(h\circ g) + \max_i \e_{\Sc_i}(h\circ g) - \min_i \e_{\Sc_i}(h\circ g) \tr 
&\leq \frac{1}{n}\sum_{i=1}^n \e_{\Sc_i}^{\rm  }(h \circ g)  +  \max_i \e_{\Sc_i}(h\circ g) - \min_i \e_{\Sc_i}(h\circ g)  \tr
&\leq \eta.
\end{align}
In other words, for any $i \in [n] = \{1, \dots, n\}$, $\e_{\Sc_i}(h\circ g) \leq \eta$ holds. 
We have from Theorem \ref{thm:parametrize} $\|h - h'\|_{1, \Dc} \leq L_\theta \d$ for any probability measure $\Dc$. Using Definition \ref{def:Lipschitz_continuous_functional} we know that $|\e_{\Sc_i}^{\rm  }(h'\circ g) - \e_{\Sc_i}^{\rm  }(h\circ g)| \leq L_{\ell} L_\theta \d$. Therefore, for any $h' \in \Hf$, we have:
\be
\e_{\Sc_i}^{\rm  }(h' \circ g) \leq \e_{\Sc_i}^{\rm  }(h \circ g) + L_{\ell} L_\theta \d \leq \eta + L_{\ell} L_\theta \d.
\en
From \eqref{eq:immediate_consequence} and Prop.~\ref{prop:new_domain}, for any $\Tc \in \conv(\Sc_1, \dots, \Sc_n)$ and any $\Sc_i$, $\Tt_{\Hf}^\rt(\Tc, \Sc_i) \leq \eta$ holds, and thus from the definition of $\Tt_{\Hf}^\rt$ we have the third inequality of
\eqref{eq:alg_result_to_prove}. The first inequality of \eqref{eq:alg_result_to_prove} follows from \Cref{prop:new_domain}. 
\end{proof}

\section{Additional theoretical results}

In this appendix we present additional theoretical results as supplementary material. 

\subsection{Necessity of excess risks}\label{app:counter_realizable}

We give an example where the realizable transfer measure is large but the source domain is transferable to the target domain.  

\begin{eg}
Consider two distributions:
\be
p_\Sc(X, Y) = \begin{cases}
0.5 & Y = 1,\,-1\leq X < 0,\\
0.5 & Y = -1,\,0\leq X < 1
\end{cases},\, p_\Tc(X, Y) = \begin{cases}
0.2 & Y = 1,\,-1\leq X < 0,\\
0.2 & Y = -1,\,0\leq X < 1 \\
0.3 & Y = 1,\,-1\leq X < 1 \\
0.3 & Y = -1,\,-1\leq X < 1
\end{cases},
\en
and the hypothesis class $\Hc$ to be the same as \Cref{eg:d_transfer_is_not_sim}. Then $\Sc$ is $(0.5\d, 0.2\d)$-transferable (Definition \ref{def:transfer}) for small $\d$. However, for any $\Hf$ that includes the optimal (source and target) classifier $h_0$ we have
\be
\mathtt{T}^\rt_{\Hf}(\Sc, \Tc) = \sup_{h\in \Hf}|\e_\Sc(h) - \e_\Tc(h)| \geq |\e_\Sc(h_0) - \e_\Tc(h_0)| = 0.3.
\en
\end{eg}


The example above shows that when the optimal errors of two domains are dissimilar, simply measuring the difference of errors cannot fully describe the transferability. Instead, we should consider the difference of the \emph{excess risks} as in \Cref{def:transfer}.

\subsection{Other IPMs}\label{app:other_ipm}

Different choices the the function class in \eqref{eq:transfer_ipm} could lead to various definitions \citep{sriperumbudur2009integral}:
\begin{itemize}
    \item maximum mean discrepancy (MMD): $\Fc_{\rm MMD} = \{f: \|f\|_{\rm Hilbert} \leq 1\}$ where the norm $\|f\|_{\rm Hilbert}$ is defined on a reproducing kernel Hilbert space (RKHS).
    \item Wasserstein distance: $\Fc_{\rm Wasserstein} = \{f: \|f\|_L \leq 1\}$ where $\|f\|_L = 1$ is the Lipschitz semi-norm of a real valued function $f$. It is also known as the Kantorovich metric. 
    \item total variation metric: $\Fc_{\rm TV} = \{\|f\|_{\infty} \leq 1\}$  where $\|f\|_{\infty} = \sup_x \{|f(x)|\}$ is the bound of $f$. This measures the total difference of the probability density functions (PDFs). 
    \item Dudley metric: $\Fc_{\rm Dudley} = \{\|f\|_{\infty} + \|f\|_L \leq 1\}$. 
    \item Kolmogorov distance: $\Fc_{\rm Kolmogorov} = \{x\mapsto \one({x\leq t}), t\in \Rb^d\}$ where we have $x \in \Rb^d$ and $x\leq t$ means that for all components we have $x_i \leq t_i$. This measures the total difference of the cumulative density functions (CDFs). 
\end{itemize}

\subsection{Estimation of transfer measures with VC dimension and Natarajan dimension}\label{app:est_vc_natarajan}

In this section, we review Rademacher complexity and show that it can be upper bounded by VC dimension \citep[e.g.][]{shalev2014understanding}. We use ${\rm VCdim}(\cdot)$ to represent the VC dimension of a function class. We also show that the estimation error in \Cref{app_lem:transfer_to_normal_error} can be upper bounded with Natarajan dimension. 
\begin{definition}[\textbf{Rademacher complexity}]
The Rademacher complexity of an i.i.d.~drawn sample set $S = \{w_i\}_{i=1}^m$, over $\Fc$ is defined as:
\begin{align}\nonumber
&\Rf_m(\Fc) = \Eb_S \left[\Eb_{\sigma_i} \sup_{f\in \Fc}\frac{1}{m}\sum_{i=1}^m \sigma_i f(w_i)\right],\mbox{ where }\\
&\{\sigma_i\}_{i=1}^m \mbox{ are independently drawn such that }Pr(\sigma_i = 1) = Pr(\sigma_i = -1) = \frac{1}{2}.\nonumber
\end{align}
\end{definition}

\begin{lem}\label{lem:upper_bound_Rademacher_VC}
Denote $d = \textrm{VCdim}(\Hf)$ where $\Hf$ is a set of functions taking values $\{-1, +1\}$. For any $m \in \Nb_+$, we have:
\be
\Rf_m(\Hf) \leq \sqrt{\frac{2}{m}\log \sum_{i=0}^d \binom{m}{i}},
\en
if $m \geq d$, then
\be
\Rf_m(\Hf) \leq \sqrt{\frac{2d}{m}\log \frac{em}{d}}.
\en
\end{lem}

\begin{proof}
This lemma follows from Corollary 3.8, Theorem 3.17 and Corollary 3.18 of \citet{mohri2018foundations}. 
\end{proof}


Combining Theorem \ref{thm:estimate_transfer_measures_app} and Lemma \ref{lem:upper_bound_Rademacher_VC}, we obtain the 
following corollary:
\begin{cor}
Suppose $\widehat{\Sc}$ and $\widehat{\Tc}$ are sample distributions of $\Sc$ and $\Tc$, with samples drawn i.i.d. Denote the sample numbers of $\widehat{\Sc}$ and $\widehat{\Tc}$ are separately $m$ and $k$.  If $\Hc$ is a set of binary classifiers with labels $\{-1, 1\}$, then for any $\Hf\subseteq \Hc$ with $d = {\rm VCdim}(\Hf)$, any of the following holds w.p.~$1 - \d$:
\begin{align}
& \Tt_\Hf(\Sc\|\Tc) \leq \Tt_\Hf(\widehat{\Sc}\| \widehat{\Tc})+ 2 \sqrt{\frac{2}{m}\log \sum_{i=0}^d \binom{m}{i}} + 2  \sqrt{\frac{2}{k}\log \sum_{i=0}^d \binom{k}{i}}  + 2\sqrt{\frac{\log(4/\d)}{2m}} + 2\sqrt{\frac{\log(4/\d)}{2k}}, \\
& \Tt_\Hf(\Sc, \Tc) \leq \Tt_\Hf(\widehat{\Sc}, \widehat{\Tc}) + 2 \sqrt{\frac{2}{m}\log \sum_{i=0}^d \binom{m}{i}} + 2 \sqrt{\frac{2}{k}\log \sum_{i=0}^d \binom{k}{i}}   + 2\sqrt{\frac{\log(4/\d)}{2m}} + 2\sqrt{\frac{\log(4/\d)}{2k}}, \\
& \Tt^{\mathtt{r}}_\Hf(\Sc, \Tc) \leq \Tt^{\mathtt{r}}_\Hf(\widehat{\Sc}, \widehat{\Tc}) + \sqrt{\frac{2}{m}\log \sum_{i=0}^d \binom{m}{i}} + \sqrt{\frac{2}{k}\log \sum_{i=0}^d \binom{k}{i}} + \sqrt{\frac{\log(4/\d)}{2m}} + \sqrt{\frac{\log(4/\d)}{2k}}.
\end{align}
If $m \geq d$ and $k\geq d$, then any of the following holds w.p.~$1 - \d$:
\begin{align}
& \Tt_\Hf(\Sc\|\Tc) \leq \Tt_\Hf(\widehat{\Sc}\| \widehat{\Tc})+ 2 \sqrt{\frac{2d}{m}\log \frac{em}{d}} + 2\sqrt{\frac{2d}{k}\log \frac{ek}{d}} + 2\sqrt{\frac{\log(4/\d)}{2m}} + 2\sqrt{\frac{\log(4/\d)}{2k}}, \\
& \Tt_\Hf(\Sc, \Tc) \leq \Tt_\Hf(\widehat{\Sc}, \widehat{\Tc}) + 2 \sqrt{\frac{2d}{m}\log \frac{em}{d}} + 2 \sqrt{\frac{2d}{k}\log \frac{ek}{d}} + 2\sqrt{\frac{\log(4/\d)}{2m}} + 2\sqrt{\frac{\log(4/\d)}{2k}}, \\
& \Tt^{\mathtt{r}}_\Hf(\Sc, \Tc) \leq \Tt^{\mathtt{r}}_\Hf(\widehat{\Sc}, \widehat{\Tc}) + \sqrt{\frac{2d}{m}\log \frac{em}{d}} + \sqrt{\frac{2d}{k}\log \frac{ek}{d}} + \sqrt{\frac{\log(4/\d)}{2m}} + \sqrt{\frac{\log(4/\d)}{2k}}.
\end{align}
\end{cor}

Moreover, if the hypothesis class $\Hc$ is the set of all possible functions that can be constructed through a fixed structure ReLU/LeakyReLU network, with $W$ the number of parameters and $L$ the number of layers, then there exists an absolute constant $C$ such that $d \leq C W L \log W$ \citep{bartlett2019nearly}.

A generalization of VC dimension is called Natarajan dimension \citep{natarajan1989learning}, which coincides with VC dimension when the classification task is binary. We have the following result \citep[Theorem 29.3]{shalev2014understanding}:
\begin{lem}
Suppose the Natarajan dimension of $\Hf$ is $d$ and the number of classes is $K$ for multiclass classification. There exists absolute constant $C$ such that for any domain $\Dc$, with probability $1 - \d$ the following holds:
\be
{\rm est}_\Hf(\Dc) = \sup_{h\in \Hf} |\e_\Dc(h) - \e_{\widehat{\Dc}}(h)| \leq C\sqrt{\frac{d \log K + \log (1/\d)}{m}}.
\en
\end{lem}

With this lemma we have the corollary:
\begin{cor}\label{cor:natarajan}
Suppose the Natarajan dimension of $\Hf$ is $d$ and the number of classes is $K$ for multiclass classification. Suppose $\widehat{\Sc}$ and $\widehat{\Tc}$ are i.i.d.~sample distributions drawn from distributions of $\Sc$ and $\Tc$, with sample number $m$ and $k$, then w.p.~at least $1 - \d$ we have:
\begin{align}
& \label{eq:estimation_one_side_transfer_measure}\Tt_\Hf(\Sc\|\Tc) \leq \Tt_\Hf(\widehat{\Sc}\| \widehat{\Tc})+ 2 C\sqrt{\frac{d \log K + \log (2/\d)}{m} } + 2C\sqrt{\frac{d \log K + \log (2/\d)}{k} }  , \\
& \Tt_\Hf(\Sc, \Tc) \leq \Tt_\Hf(\widehat{\Sc}, \widehat{\Tc}) + 2 C\sqrt{\frac{d \log K + \log (2/\d)}{m}} + 2C\sqrt{\frac{d \log K + \log (2/\d)}{k} } , \\
& \Tt^{\mathtt{r}}_\Hf(\Sc, \Tc) \leq \Tt^{\mathtt{r}}_\Hf(\widehat{\Sc}, \widehat{\Tc}) +  C\sqrt{\frac{d \log K + \log (2/\d)}{m}} + C\sqrt{\frac{d \log K + \log (2/\d)}{k}}.
\end{align}
\end{cor}
\begin{proof}
This proof is similar to the proof of \Cref{thm:estimate_transfer_measures_app}, using union bound as in \eqref{eq:union_bound_rademacher}.
\end{proof}
Estimation of Natarajan dimension can be found in \citet{natarajan1989learning, shalev2014understanding}.

\subsection{Functional point of view of surrogate loss}\label{app:functional_surrogate}

In this appendix we study the Lipschitzness and strong convexity of the surrogate loss, especially cross entropy. We use the terms distribution and measure interchangeably, since distributions can be treated as probability measures.  

\subsubsection{Lipschitz continuity of loss}\label{app:lipschitz_cont_loss}

Let define the $L_p$ distance ($p\geq 1$) (e.g.~\citet{rudin1987real}) between two functions:
\be
\|h - h'\|_{p, \mu} = \left(\int \|h(x) - h'(x)\|_2^p d\mu\right)^{1/p},
\en
where $\mu$ is a measure. We consider the following definition of Lipschitz functional:
\begin{definition}[\textbf{Lipschitz continuity}]\label{def:Lipschitz_continuous_functional}
A functional $h\mapsto f(h)$ that maps a function to a real number is $f$ is Lipschitz continuous on $\Hc$ w.r.t.~measure $\mu$ if there exists an absolute constant $L$ such that:
\be\label{eq:Lipschitz_continuous}
|f(h) - f(h')| \leq L \|h - h'\|_{1, \mu}
\en
for all function $h, h'\in \Hc$. 
\end{definition}

One can show that the cross entropy loss is a Lipschitz continuous functional with mild assumptions:
\begin{prop}[]
For binary classification with labels $\{-1, +1\}$, suppose $\Hc$ is a hypothesis class whose elements satisfy $h: \Xc \to (-1 + \d, 1 - \d)$ with $0 < \d < 1$, then $\e_\Dc^{\rm CE}$ is $(\log 2)^{-1}\d^{-1}$ Lipschitz continuous w.r.t.~any distribution $\Dc$. Furthermore, for multi-class classification, suppose $\Yc = \{1, 2, \dots, K\}$ and the prediction $h(x)$ is a $K$-dimensional probability vector on the simplex. If $\Hc$ is a hypothesis class whose elements satisfy $h_i(x) \geq \d$ for all $i \in \Yc$ and $x\in \Xc$, then $\e_\Dc^{\rm CE}$ is $(\log 2)^{-1}\d^{-1}$ Lipschitz continuous w.r.t.~any distribution $\Dc$. 
\end{prop}

Note that a simplex is defined as: $\{\pi \in \Rb^d: \one^\top \pi = 1, \pi_i \geq 0\}$, where $\pi$ is called a probability vector. Before we move on to the proof, we can show that the assumption of $h$ is often satisfied in practice. For binary classification, the widely used tanh/sigmoid function can guarantee that the value of $h$ is never exactly $-1$ or $1$. For multiclass classification, the softmax function guarantees that $h_i(x) > 0$ for all $i$ and $x\in \Xc$. If the input space is bounded and $h$ is continuous, then $h_i(x) \geq \d$ for all $i$ and $x\in \Xc$.
\begin{proof}
For binary classification we have:
\begin{align}
\e_{\Dc}^{\rm CE}(h) &= \int p_\Dc(x, 1) \ell^{\rm CE}(h(x), 1) + p_\Dc(x, -1) \ell^{\rm CE}(h(x), -1) dx \tr
&= \int -p_\Dc(x, 1) \log_2 \frac{1+h(x)}{2} - p_\Dc(x, -1) \log_2 \frac{1-h(x)}{2} dx.
\end{align}
Therefore, with the mean value theorem we have:
\begin{align}
|\e_{\Dc}^{\rm CE}(h) - \e_{\Dc}^{\rm CE}(h')| &= (\log 2)^{-1} \left|\int (h(x) - h'(x))\left( \frac{-p_\Dc(x, 1)}{1 + h_{\xi}(x)} + \frac{p_\Dc(x, -1)}{1 - h_{\xi}(x)}\right)dx \right|,\tr
&\leq (\log 2)^{-1} \int |h(x) - h'(x)|\left|\frac{-p_\Dc(x, 1)}{1 + h_{\xi}(x)} + \frac{p_\Dc(x, -1)}{1 - h_{\xi}(x)}\right| dx \tr
&\leq (\log 2)^{-1} \int |h(x) - h'(x)|\left(\left|\frac{-p_\Dc(x, 1)}{1 + h_{\xi}(x)} \right| +\left| \frac{p_\Dc(x, -1)}{1 - h_{\xi}(x)}\right|\right) dx \tr
&\leq (\log 2)^{-1} \int |h(x) - h'(x)| \d^{-1}(p_\Dc(x, 1) + p_\Dc(x, -1)) dx \tr
&= (\log 2)^{-1} \d^{-1}\|h - h'\|_{1, \Dc}.
\end{align}
where in the first line $h_\xi(x) = (1 - \xi(x))h(x) + \xi(x) h'(x)$ is a (pointwise) convex combination of $h(x)$ and $h'(x)$ with $0\leq \xi(x) \leq 1$; in the third line we used triangle inequality; in the fourth line we use the condition that the values of $h, h'$ are in the region $(-1 + \d, 1 - \d)$. 

Similarly, for multiclass classification with $K$ classes, the ground truth $y$ is a one-hot $K$-dimensional vector, and the prediction $h(x)$ is a $K$-dimensional vector on a simplex. The cross entropy loss is:
\begin{align}
\e_\Dc^{\rm CE}(h) &= \sum_y \int \ell^{\rm CE}(h(x), y) p(x, y) dx = \sum_y \int -y\cdot \log_2 h(x) p(x, y) dx.
\end{align}
Similarly, we have:
\begin{align}
|\e_\Dc^{\rm CE}(h) - \e_\Dc^{\rm CE}(h')| &= (\log 2)^{-1}\left| \sum_y \int -y \cdot \frac{h(x) - h'(x)}{h_{\xi}(x)}p_\Dc(x, y) dx \right| \tr
&\leq  (\log 2)^{-1}\sum_y \int \left|-y \cdot \frac{h(x) - h'(x)}{h_{\xi}(x)}\right| p_\Dc(x, y)  dx  \tr
&\leq (\log 2)^{-1}\d^{-1} \sum_y \int \|y\|_2\cdot \|h(x) - h'(x)\|_2p_\Dc(x, y)dx \tr
&= (\log 2)^{-1}\d^{-1}\|h - h'\|_{1, \Dc},
\end{align}
where in the first line we use the mean value theorem and $h_\xi(x) = (1 - \xi(x))h(x) + \xi(x) h'(x)$ is a (pointwise) convex combination of $h(x)$ and $h'(x)$ with $0\leq \xi(x) \leq 1$; also in the first line we define $(h(x) - h'(x))/h_\xi(x)$ to be a vector with each component to be $(h_i(x) - h'_i(x))/h_\xi(x)_i$; in the third line we use Cauchy--Schwarz inequality and that $h_i(x) \geq \d$, $h'_i(x) \geq \d$ for any $i$ and any $x\in \Xc$. 
\end{proof}

\subsubsection{Strongly convex functional}

So far, we have seen that for Lipschitz continuous loss, if the change of $h$ is small, then the change of loss $\e_\Dc(h)$ is also small. Now we ask if the converse is true. This is important to characterize the $\d$-minimal set (the set of approximately optimal classifiers). We first define strongly convex functional:
\begin{definition}\label{def:strongly_cvx_functional}
A functional $f:\Hc \to \Rb$ is $\lambda$-strongly convex on a convex set $\Hc$ w.r.t.~measure $\mu$ if for any $h, h'\in \Hc$ and $\a \in [0, 1]$, we have:
\be
f(\a h + (1 - \a) h') \leq \a f(h) + (1 - \a) f(h') - \frac{\lambda}{2}\a(1 - \a) \|h - h'\|_{2, \mu}^2,
\en
where we defined the $L_2$ norm of a function:
\be
\|h - h'\|_{2, \mu} := \left(\int \|h - h'\|_2^2 d\mu\right)^{1/2}.
\en
\end{definition}

We use $L_2$ norm because it can translate the strong convexity of the loss functional to the strong convexity of the loss function $\ell(\cdot, y)$ easily:
\begin{lem}\label{lem:loss_to_functional}
Given a convex hypothesis class $\Hc$, suppose that $\ell$ is $\lambda$-strongly convex in the first argument, i.e.~for any $y \in \Yc$,  $\hy_1, \hy_2$ and $\a \in [0, 1]$ we have:
\be
\ell(\a \hy_1 + (1 - \a) \hy_2, y) \leq \a \ell(\hy_1, y) + (1 -\a) \ell(\hy_2, y) - \frac{\lambda}{2}\a(1 - \a) \|\hy_1 - \hy_2\|_2^2, 
\en
then the loss functional 
\be
\e_\Dc(h) = \sum_y \int p_\Dc(x, y) \ell(h(x), y) dx
\en
is also $\lambda$-strongly convex w.r.t.~measure $\Dc$. 
\end{lem}

\begin{proof}
Straightforward by plugging in Definition \ref{def:strongly_cvx_functional}. 
\end{proof}

For cross entropy loss, we have the following:

\begin{cor}
For binary classification, cross entropy risk functional $\e_\Dc^{\rm CE}$ is $(4 \log 2)^{-1}$-strongly convex on $\Dc$ and $(\log 2)^{-1}$-strongly convex on $\Dc$ for multiclass classification.
\end{cor}

\begin{proof}
For binary classification, we have:
\be
\ell^{\rm CE}(\hy, 1) = -\log_2 \frac{1 + \hy}{2}, \, \ell^{\rm CE}(\hy, -1) = -\log_2 \frac{1 - \hy}{2},
\en
which are both $(4\log 2)^{-1}$-strongly convex on $\hy \in (-1, 1)$. For multiclass classification, we have:
\be
\ell^{\rm CE}(\hy, y) = -\log_2 \hy_i, 
\en
for any unit one-hot vector $y = e_i$ ($e_i$ is the $i^{\rm th}$ element of standard basis in $\Rb^K$). This is $(\log 2)^{-1}$-strongly convex for $\hy_i \in (0, 1)$. The rest follows from Lemma \ref{lem:loss_to_functional}.
\end{proof}

From the strongly convexity we can derive the uniqueness of the function (up to $L_2$ norm) and relate $\d$-minimal set to an $L_2$ neighborhood of an optimal classifier. 

\begin{thm}
For any $\lambda$-strongly convex functional $f$ on a convex hypothesis class $\Hc$ w.r.t.~measure $\mu$, the minimizer is almost surely unique, in the sense that if $h_1^*$, $h_2^*$ are both minimizers, then 
\be
\|h_1^* - h_2^*\|_{2, \mu} = 0, 
\en
and thus $h_1^*$, $h_2^*$ only differ by a measure zero set. Suppose $h^* \in \argmin f(h)$. If $f(h) \leq f^* + \e$ with $f^*$ the optimal value, then 
\be\label{eq:function_set}
\|h - h^*\|_{2, \mu} \leq \sqrt{\frac{2}{\lambda}\e}.
\en
\end{thm}

\begin{proof}
It suffices to prove the second claim only. From the definition of strong convexity, for $\a\in [0, 1]$ we have:
\begin{align}
f^* \leq f(\a h^* + (1 - \a) h) &\leq \a f(h^*) + (1 - \a) f(h) - \frac{\lambda}{2}\a(1 - \a) \|h^* - h\|_{2, \mu}^2 \tr
&\leq (1 - \a) \e + f^* - \frac{\lambda}{2}\a(1 - \a) \|h^* - h\|_{2, \mu}^2,
\end{align}
where we use $h^* \in \argmin f(h)$ and $f(h) \leq f^* + \e$. From this inequality we obtain that:
\be
\|h - h^*\|_{2, \mu}^2 \leq \frac{2 \e}{\a \lambda}.
\en
By taking $\a \to 1$ we obtain \eqref{eq:function_set}.
\end{proof}

With this theorem we can characterize the $\d$-minimal set $\argmin(\e_\Dc, \d)$ as some neighborhood of the unique optimal classifier $h^*$, if the functional $\e_\Dc$ is strongly convex and Lipschitz continuous. Symbolically, it can be represented as:
\be\label{eq:minimal_set_norm_ball}
\Bc_2(h^*) \subseteq \argmin(\e_\Dc, \d) \subseteq  \Bc_1(h^*),
\en
where $\Bc_p(h^*)$ is some $L_p$ norm ball with the center $h^*$. 

\subsubsection{Parametric formulation of classifier}\label{sec:parametrize}

We discussed the $L_p$ distance between functions in previous subsections. In practice the functions are often parametrized:
\be
h(x) = q(\theta, x).
\en
One can show that $L_p$ distances between two functions $h = q(\theta, \cdot)$ and $h' = q(\theta', \cdot)$ on the function space can be upper bounded:
\begin{thm}\label{thm:parametrize}
Suppose $h = q(\theta, \cdot)$ is parameterized by $\theta$ and for any $x\in \Xc$, $q(\cdot, x)$ is $L$-Lipschitz continuous (w.r.t.~$\ell_2$ norm), then for any $1 \leq p < \infty$ and probability measure $\mu$ we have:  
\be
\|h - h'\|_{p, \mu} \leq L \|\theta - \theta'\|_2.
\en
\end{thm}

\begin{proof}
From the Lipschitz continuity we can derive:
\begin{align}
\|h - h'\|_{p, \mu} &= \left(\int \|h(x) - h'(x)\|_2^p d\mu\right)^{1/p}  \tr
&= \left(\int \|q(\theta, x) - q(\theta', x)\|_2^p d\mu\right)^{1/p}  \tr
&\leq \left(\int (L\|\theta - \theta'\|_2)^p d\mu\right)^{1/p} \tr
&= L\|\theta - \theta'\|_2.
\end{align}
\end{proof}

The theorem above tells us that in parametrized models the closeness in terms of parameters can imply the closeness in terms of the model function. However, the converse may not be true. For example, we can permute hidden neurons of the same layer in a neural network and obtain the same function, but the parametrization can be drastically different.

\subsection{Comparison with other frameworks} \label{app:comparison}

We compare our \Cref{alg:train_transfer} with existing adversarial training frameworks. 

\paragraph{Distributional robustness optimization (DRO)}

\citet{sinha2017certifying} proposed a distributional robustness framework for generalizing to unseen domains. In this framework, the following minimax problem is proposed:
\be
\min_{g, h}\max_{\Sc': W(\Sc', \Sc) \leq \d} \epsilon_\Sc'(h\circ g),
\en
which says that the classification error is small for any distribution $\Sc'$ close to our original source distribution $\Sc$. Here $W(\cdot, \cdot)$ denotes the Wasserstein metric. As we have discussed in Example \ref{eg:d_transfer_is_not_sim}, transferability does not necessarily mean that the distributions have to be close.

\paragraph{DANN} The Domain Adversarial Neural Network (DANN) formulation \citep{ganin2015unsupervised} solves the following minimax optimization problem:
\begin{align}
\min_{g, h} \max_{h'} \e_\Sc(h\circ g) + \Eb_{x\sim p_\Sc|_x} [\log (h'\circ g)(x)] +  \Eb_{x\sim p_\Tc|_x} [\log (1 - (h'\circ g)(x))],
\end{align}
where $g$ is a feature embedding, $h$ is a classifier and $h'$ is a domain discriminator. 
If we can solve the inner maximization problem exactly, then we obtain the Jensen--Shannon divergence between the push-forwards of the input distributions $g\# p_S|_x$ and $g\# p_T|_x$. In other words, we want to obtain a feature embedding $g$ and a classifier $h$ such that:
\be
\e_\Sc(h\circ g) + D_{\rm JS}((g\# p_{\Sc}|_x) \| (g\# p_{\Tc}|_x)),
\en
is minimized, with $D_{\rm JS}$ denoting the Jensen--Shannon divergence. On the one hand, we need to have small classification error given the feature embedding $g$. On the other hand, the feature embedding between source and target should be similar. Our framework is similar to DANN in the sense that they both solve minimax problems. The difference is that we minimize the transfer measure which is weaker than the similarity between distributions (\Cref{eg:d_transfer_is_not_sim}).

\paragraph{$\Hc\Delta\Hc$-divergence} Finally we prove that our transfer measure is tighter than $\Hc\Delta\Hc$-divergence \citep{blitzer2007learning}. We rewrite the theoretical result regarding $\Hc\Delta\Hc$-divergence:
\begin{thm}[Theorem 1, \cite{blitzer2007learning}]\label{thm:h-divergence}
Let $\lambda^* = \argmin_{h\in \Hc}(\e_\Tc(h) + \e_\Sc(h))$, and the $\Hc\Delta\Hc$-divergence between the input marginal distributions $\Sc|_x$ and $\Tc|_x$ to be $d_{\Hc\Delta \Hc}(\Sc|_x, \Tc|_x)$, then for binary classification and for any $h\in \Hc$ we have:
\begin{align}
\e_\Tc(h) \leq \e_\Sc(h) + \lambda^* + \frac{1}{2} d_{\Hc\Delta \Hc}(\Sc|_x, \Tc|_x).
\end{align}
\end{thm}
Now let us prove that our \Cref{thm:target_bound} is tighter than \Cref{thm:h-divergence}:
\begin{prop}\label{prop:compare_trans_measure_hdiv}
The target error bound with our transfer measure $\Tt_{\Gamma}(\Sc\| \Tc)$ is tighter than the target error bound with $\Hc\Delta\Hc$-divergence, i.e., for any $h\in \Hf$ we have:
\begin{align}\label{eq:transfer_measure_h_div}
\e_\Tc(h) \leq \e_\Sc(h) + \e_{\Tc}^* - \e_{\Sc}^* + \Tt_{\Hf}(\Sc \| \Tc)  \leq \e_\Sc(h) + \lambda^* + \frac{1}{2} d_{\Hc\Delta \Hc}(\Sc|_x, \Tc|_x).
\end{align}
\end{prop}

\begin{proof}
Note that from \Cref{def:transfer_measure} we can rewrite the middle of \eqref{eq:transfer_measure_h_div} as $\sup_{h\in \Hf}(\epsilon_\Tc(h) - \e_\Sc(h))$. Suppose $h^* \in \argmax_{h\in \Hf}(\epsilon_\Tc(h) - \e_\Sc(h))$, then from \Cref{thm:h-divergence} we have:
\begin{align}
\e_\Tc(h^*) \leq \e_\Sc(h^*) + \lambda^* + \frac{1}{2} d_{\Hc\Delta \Hc}(\Sc|_x, \Tc|_x),
\end{align}
and thus:
\begin{align}
 \e_{\Tc}^* - \e_{\Sc}^* + \Tt_{\Hf}(\Sc \| \Tc) =  \sup_{h\in \Hf}(\epsilon_\Tc(h) - \e_\Sc(h)) = \epsilon_\Tc(h^*) - \e_\Sc(h^*) \leq \lambda^* + \frac{1}{2} d_{\Hc\Delta \Hc}(\Sc|_x, \Tc|_x).
\end{align}
\end{proof}

\section{Additional Experiments}\label{app:add_exp}

We present additional experimental details in this section. 




\subsection{Datasets}\label{app:datasets}

The four datasets in this paper are RotatedMNIST \citep{ghifary2015domain}, PACS \citep{li2017deeper},  Office-Home \citep{venkateswara2017deep} and WILDS-FMoW \citep{koh2020wilds}.
Here is a short description:

\begin{itemize}
\item RotatedMNIST: this dataset is an adaptation of MNIST. It has six domains, and each domain rotates the images in MNIST with a different angle. The angles are $\{0\degree, 15\degree, 30\degree, 45\degree, 60\degree, 75\degree\}$. We choose the domain with $0\degree$ to be the target domain and the rest to be the source domains. Each image is grayscale and has $28\times 28$ pixels. The label set is $\{0, 1, \dots, 9\}$. The numbers of images of each domain are 11667, 11667, 11667, 11667, 11666, 11666. The total is 70000.
\item PACS: this dataset has four domains: photo (P), art painting (A), cartoon (C) and sketch (S). Each image is RGB colored and has $224\times 224$ pixels. There are 7 categories in total and 9991 images. The number of images of each domain: A: 2048; C: 2344; P: 1670; S: 3929. We choose the art painting domain to be the target domain and the rest to be the source domains. 
\item  Office-Home: this dataset has four domains: Art, Clipart, Product, Real-World. Each image is RGB colored and has $224\times 224$ pixels. There are 65 categories and 15588 images in total. The numbers of images of each domain: Art: 2427, Clipart: 4365, Product: 4439, Real-World: 4357. We choose the Art domain to be the target domain and the rest to be the source domains. 

\item WILDS-FMoW: WILDS \citep{koh2020wilds} is a benchmark for domain generalization including several datasets. The Functional Map of the World (FMoW) is one of them, which is a variant of \citet{christie2018functional}. Each image is RGB colored and has $224\times 224$ pixels. There are 62 categories and 469835 images in total. There six domains in total and we choose five of them, since the last domain has too few images. The numbers of images of each domain are 103299, 162333, 33239, 157711, 13253, and we choose the last domain as the target domain. The rest are source domains. The license can be found at \url{https://wilds.stanford.edu/datasets/}. 
\end{itemize}

\subsection{Experimental settings}\label{app:setting}

We introduce the experimental settings in this subsection. The code is modified from \url{https://github.com/facebookresearch/DomainBed}, with the license in \url{https://github.com/facebookresearch/DomainBed/blob/master/LICENSE}. 
\begin{itemize}
\item Hardware: Our experiments are run on a cluster of GPUs, including NVIDIA RTX6000, T4 and P100. 
\item Datasplit: we use the same data split as in \citet{gulrajani2020search} except the WILDS-FMoW dataset, where we throw away the last region because it has only very few samples (201 samples). For all datasets we use data augmentation.

\item Batch size: for all experiments on RotatedMNIST we choose batch size 64, for  Office-Home and PACS we choose batch size 32 (for our Transfer algorithm and PACS we choose batch size 16), and for WILDS-FMoW we choose batch size 16. In each epoch, we go through $k$ steps, where $k$ is the smallest number of samples among domains, divided by the batch size.
\item Optimization: for the training of all other algorithms different from our Transfer Algorithm, we use the default setting from \citet{gulrajani2020search}. We choose Adam as the default optimizer for training, with learning rate 1e-3 for RotatedMNIST, and learning rate 5e-5 for other datasets. For RotatedMNIST, PACS and  Office-Home we run for $5000$ steps; For WILDS-FMoW we run for $50000$ steps. 
\item Neural Architecture: we use the same neural architecture as in \citet{gulrajani2020search}. For each dataset, the feature embedding and classifier architectures for all algorithms are the same. Specifically, all classifiers are linear layers. For RotatedMNIST the feature embedding is CNN with batch normalization and for other datasets the feature embedding is ResNet50. 
\item Algorithm \ref{alg:measure_transfer}: we choose Adam optimizer with projection. The learning rates are the same as the training algorithms: for RotatedMNIST we choose 1e-3, and we choose 5e-5 for others. We run the algorithm for 10 epochs and for three independent trials. Among the three trials, we choose the accuracies with the largest gap between the target domain and one of the source domains. The source domain is chosen in such a way that the gap is the largest among all source domains.
\item Algorithm \ref{alg:train_transfer} optimization: for RotatedMNIST we run Adam for minimization with learning rate 0.01 and Stochastic Gradient Ascent (SGA) for maximization with learning rate 0.01. We choose the ascent steps to be 30 for each inner loop and the projection radius to be $\delta = 10.0$; for PACS we run Adam for minimization with learning rate 5e-5 and Stochastic Gradient Ascent (SGA) for maximization with learning rate 0.001. We choose the ascent steps to be 30 for each inner loop and the projection radius to be $\delta = 0.3$; for  Office-Home dataset we load the pretrained model from SD, and run Stochastic Gradient Descent Ascent with learning rate $0.001$ and $\delta = 0.3$, i.e., each inner loop takes only one step of SGA and each outer loop takes one step of SGD; for WILDS-FMoW dataset we loaded the pretrained model from ERM, and run SGA for 20 steps in each inner loop, with learning rate $0.001$ and $\delta = 0.5$, for each outer loop we run SGD with $\texttt{lr}=0.001$. 

\item Step number for Algorithm \ref{alg:train_transfer}: for RotatedMNIST and PACS we train for $8000$ outer steps with each outer step including $30$ inner steps. For  Office-Home we train for $5000$ outer steps with each outer step including one inner step; for WILDS-FMoW we train for $5000$ outer loops with each outer step including $20$ inner steps. 
\end{itemize}

\subsection{Additional results}

We present additional experiments on RotatedMNIST \citep{ghifary2015domain}, PACS \citep{li2017deeper} and  Office-Home \citep{venkateswara2017deep}. Thanks to the suite from \citet{gulrajani2020search}, we are able to compare a wide array of algorithms under the same settings. The algorithms that we compare include
\begin{itemize}
\item Empirical Risk Minimization \citep[ERM,][]{vapnik1992principles}
\item Invariant Risk Minimization \citep[IRM,][]{arjovsky2019invariant}
\item Domain Adversarial Neural Network \citep[DANN, ][]{ganin2015unsupervised}
\item Conditional DANN \citep[CDANN,][]{li2018deep}
\item Correlation Alignment \citep[CORAL,][]{sun2016deep}
\item Maximum Mean Discrepancy \citep[MMD,][]{li2018domain2}
\item Variance Risk Extrapolation \citep[VREx,][]{krueger2020out}
\item Mariginal Transfer Learning \citep[MTL,][]{blanchard2021domain}
\item Spectral Decoupling \citep[SD,][]{pezeshki2020gradient}
\item Meta Learning Domain Generalization \citep[MLDG,][]{li2018learning}
\item Mixup \citep{xu2020adversarial, yan2020improve}
\item Representation Self-Challenging \citep[RSC,][]{huang2020self}
\item Group Distributionally Robust Optimization \citep[GroupDRO,][]{Sagawa*2020Distributionally}
\item  Style-Agnostic Network \citep[SagNet,][]{nam2019reducing}
\end{itemize}

\subsubsection{RotatedMNIST}

In Figure \ref{fig:measure_transfer} we show the performance of various algorithms on RotatedMNIST, including ERM, IRM, DANN, CORAL, MMD, VREx, MTL, SD and our Transfer algorithm. It can be seen that many algorithms fail our attack. For instance, based on the learned features, MTL classifies a source domain with $\sim$95\% (at $\delta = 3.0$) but the target accuracy drops by $\sim$20\%.

We also compare our Transfer algorithm (Algorithm \ref{alg:train_transfer}) with different hyperparameters.  
From Figure \ref{fig:RMNIST_transfer_alg} we can see that for RotatedMNIST, taking more inner steps (per outer step) has better performance.

Finally, we present results from Algorithm \ref{alg:measure_transfer} with information about losses and accuracies, for a wide array of algorithms in Table \ref{tab:evaluation_current_transfer}. 
\begin{figure}[ht]
    \centering
    \includegraphics[width=0.8\textwidth]{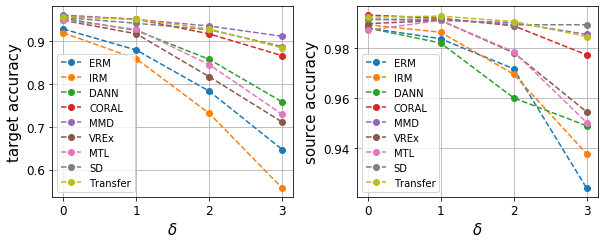}
    \caption{Measuring the transferability of various algorithms for domain generalization on RotatedMNIST. For the Transfer algorithm we take $\d = 10.0$ and the number of ascent steps to be $30$. }
    \label{fig:measure_transfer}
\end{figure}

\begin{figure}[ht]
    \centering
    \includegraphics[width=0.8\textwidth]{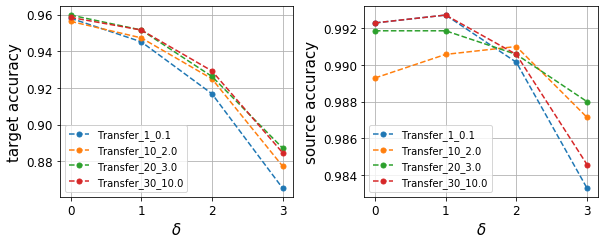}
    \caption{Evaluation of transferability of popular algorithms for domain generalization on Transfer algorithm with different hyperparameters. The dataset is RotatedMNIST. For the ascent method we use SGD with learning rate $0.01$, and the descent method to be Adam with learning rate $10^{-3}$. ``\texttt{Transfer\_$d$\_$\delta$}'' means the inner loop takes $d$ steps with the radius $\d$.}
    \label{fig:RMNIST_transfer_alg}
\end{figure}

\begin{table}[]
\caption{Evaluation of transferability of popular algorithms for domain generalization on RotatedMNIST. \textbf{algorithm}: the model that we evaluate; \textbf{$\bm{\delta}$}: the adversarial radius $\d$ we choose in Algorithm \ref{alg:measure_transfer_whole}; \textbf{max/min index}: the index of the domain with the maximal/minimal (test) classification errors (w.r.t.~0-1 loss), and index $0$ denotes the target domain; \textbf{max/min loss}: the largest/smallest loss among domains (including the target domain); \textbf{worst/best acc}: the smallest/largest classification test accuracies among domains (including the target domain). All the algorithms are using the same architectures for the feature embedding and the classifier. }
\centering
\begin{tabular}{c c c c c c c c c}
\bf algorithm & $\bm{\delta}$ & \bf max index &\bf  min index & \bf max loss & \bf min loss & \bf worst acc & \bf best acc 
\\
\hline \hline 
\rowcolor{lightgray}
\rowcolor[gray]{0.9}
ERM & 0.0 & 0 & 4 & 0.229 & 0.003 & 92.93\% & 98.80\%\\
\rowcolor[gray]{0.9}
ERM & 2.0 & 0 & 4 & 0.975 & 0.083 & 78.61\% & 97.17\%\\
\rowcolor{lightgray}
GroupDRO & 0.0 & 0 & 4 & 0.136 & 0.000 & 95.76\% & 99.27\%\\
\rowcolor{lightgray}
GroupDRO & 2.0 & 0 & 4 & 0.370 & 0.015 & 84.48\% & 98.07\%\\
\rowcolor[gray]{0.9}
SagNet & 0.0 & 0 & 4 & 0.109 & 0.000 & 96.61\% & 99.36\%\\
\rowcolor[gray]{0.9}
SagNet & 2.0 & 0 & 4 & 0.222 & 0.008 & 91.30\% & 98.67\%\\
\rowcolor{lightgray}
IRM & 0.0 & 0 & 4 & 0.578 & 0.263 & 81.87\% & 92.20\%\\
\rowcolor{lightgray}
IRM & 2.0 & 0 & 4 & 1.759 & 0.637 & 46.29\% & 86.76\%\\
\rowcolor[gray]{0.9}
DANN & 0.0 & 0 & 5 & 0.136 & 0.014 & 95.41\% & 98.29\%\\
\rowcolor[gray]{0.9}
DANN & 2.0 & 0 & 5 & 0.441 & 0.098 & 85.81\% & 96.19\%\\
\rowcolor{lightgray}
ARM & 0.0 & 0 & 4 & 0.145 & 0.002 & 95.76\% & 99.10\%\\
\rowcolor{lightgray}
ARM & 2.0 & 0 & 4 & 0.523 & 0.047 & 84.23\% & 98.54\%\\
\rowcolor[gray]{0.9}
Mixup & 0.0 & 0 & 4 & 0.175 & 0.009 & 94.98\% & 99.36\%\\
\rowcolor[gray]{0.9}
Mixup & 2.0 & 0 & 4 & 0.701 & 0.035 & 73.98\% & 98.71\%\\
\rowcolor{lightgray}
CORAL & 0.0 & 0 & 4 & 0.119 & 0.001 & 95.93\% & 99.31\%\\
\rowcolor{lightgray}
CORAL & 2.0 & 0 & 4 & 0.230 & 0.005 & 91.77\% & 98.89\%\\
\rowcolor{lightgray}
CORAL & 3.0 & 0 & 4 & 0.372 & 0.056 & 86.67\% & 97.73\%\\
\rowcolor[gray]{0.9}
MMD & 0.0 & 0 & 3 & 0.125 & 0.005 & 96.19\% & 99.14\%\\
\rowcolor[gray]{0.9}
MMD & 2.0 & 0 & 3 & 0.199 & 0.014 & 93.61\% & 99.01\%\\
\rowcolor[gray]{0.9}
MMD & 3.5 & 0 & 3 & 0.300 & 0.036 & 89.54\% & 97.86\%\\
\rowcolor{lightgray}
RSC & 0.0 & 0 & 4 & 0.146 & 0.000 & 95.46\% & 99.31\%\\
\rowcolor{lightgray}
RSC & 1.0 & 0 & 4 & 0.360 & 0.007 & 89.33\% & 98.71\%\\
\rowcolor{lightgray}
RSC & 2.0 & 0 & 4 & 1.343 & 0.289 & 72.01\% & 92.11\%\\
\rowcolor[gray]{0.9}
VREx & 0.0 & 0 & 5 & 0.137 & 0.003 & 94.94\% & 98.97\%\\
\rowcolor[gray]{0.9}
VREx & 2.0 & 0 & 5 & 0.551 & 0.082 & 81.74\% & 97.81\%\\
\rowcolor{lightgray}
CDANN & 0.0 & 0 & 5 & 0.121 & 0.010 & 95.97\% & 98.76\%\\
\rowcolor{lightgray}
CDANN & 2.0 & 0 & 5 & 0.410 & 0.079 & 84.78\% & 95.67\%\\
\rowcolor[gray]{0.9}
MLDG & 0.0 & 0 & 5 & 0.151 & 0.000 & 95.63\% & 98.89\%\\
\rowcolor[gray]{0.9}
MLDG & 2.0 & 0 & 5 & 0.351 & 0.006 & 88.90\% & 98.76\%\\
\rowcolor{lightgray}
MTL & 0.0 & 0 & 4 & 0.150 & 0.000 & 94.98\% & 99.44\%\\
\rowcolor{lightgray}
MTL & 2.0 & 0 & 4 & 0.417 & 0.014 & 84.57\% & 98.20\%\\
\rowcolor[gray]{0.9}
SD & 0.0 & 0 & 2 & 0.250 & 0.092 & 95.63\% & 99.01\%\\
\rowcolor[gray]{0.9}
SD & 2.0 & 0 & 2 & 0.630 & 0.490 & 92.76\% & 98.97\%\\
\rowcolor[gray]{0.9}
SD & 3.0 & 0 & 2 & 1.070 & 0.937 & 88.81\% & 98.33\%\\
\hline
\end{tabular}
\label{tab:evaluation_current_transfer}
\end{table}

\subsubsection{PACS}

We implement similar experiments on PACS. Figure \ref{fig:measure_transfer_PACS} and \Cref{tab:evaluation_transfer_pacs} show the results of Algorithm \ref{alg:measure_transfer}. \Cref{fig:transfer_alg_PACS} shows that taking more inner steps has better performance.

\begin{figure*}[ht]
    \centering
    \includegraphics[width=0.8\textwidth]{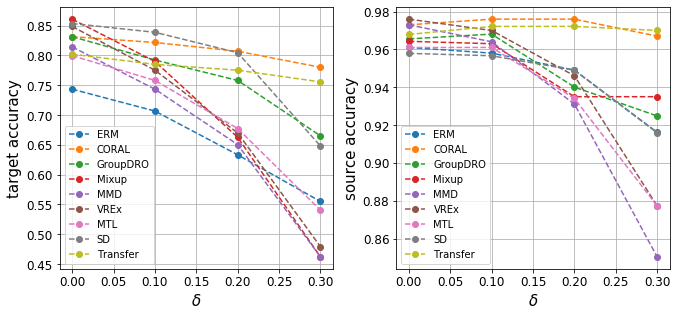}
    \caption{Measuring the transferability of various algorithms for domain generalization on PACS dataset. For the Transfer algorithm we choose $\delta = 0.3$, batch size 16, the number of ascent steps to be 30 using SGD with learning rate $0.001$. }
    \label{fig:measure_transfer_PACS}
\end{figure*}

\begin{figure}
    \centering
    \includegraphics[width=0.8\textwidth]{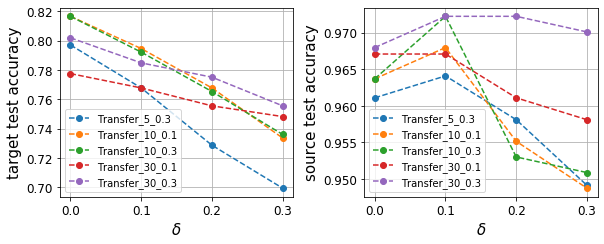}
    \caption{Evaluation of transferability of popular algorithms for domain generalization on Transfer algorithm with different hyperparameters. The dataset is PACS. One can see that if the number of inner steps is large and $\delta$ is large, then the classifier is more robust. ``\texttt{Transfer\_$d$\_$\delta$}'' means the inner loop takes $d$ steps with the radius $\d$.}
    \label{fig:transfer_alg_PACS}
\end{figure}

\begin{table}[]
\caption{Evaluation of transferability of popular algorithms for domain generalization on PACS. \textbf{algorithm}: the model that we evaluate; \textbf{$\bm{\delta}$}: the adversarial radius $\d$ we choose in Algorithm \ref{alg:measure_transfer_whole}; \textbf{max/min index}: the index of the domain with the maximal/minimal (test) classification errors (w.r.t.~0-1 loss); \textbf{max/min loss}: the largest/smallest loss among domains (including the target domain); \textbf{worst/best acc}: the smallest/largest classification test accuracies among domains (including the target domain). }
\centering
\begin{tabular}{c c c c c c c c c}
\bf algorithm & $\bm{\delta}$ & \bf max index &\bf  min index & \bf max loss & \bf min loss & \bf worst acc & \bf best acc \\
\hline \hline
\rowcolor[gray]{0.9}
ERM & 0.0 & 0 & 2 & 1.327 & 0.011 & 74.33\% & 96.11\%\\
\rowcolor[gray]{0.9}
ERM & 0.2 & 0 & 2 & 2.449 & 0.064 & 63.33\% & 94.91\%\\
\rowcolor{lightgray}
GroupDRO & 0.0 & 0 & 2 & 0.820 & 0.012 & 83.13\% & 97.60\%\\
\rowcolor{lightgray}
GroupDRO & 0.2 & 0 & 2 & 1.509 & 0.052 & 75.79\% & 95.81\%\\
\rowcolor[gray]{0.9}
SagNet & 0.0 & 0 & 2 & 0.919 & 0.002 & 77.51\% & 99.10\%\\
\rowcolor[gray]{0.9}
SagNet & 0.1 & 0 & 2 & 1.409 & 0.014 & 71.39\% & 97.01\%\\
\rowcolor[gray]{0.9}
SagNet & 0.2 & 0 & 2 & 2.002 & 0.094 & 60.64\% & 94.31\%\\
\rowcolor{lightgray}
Mixup & 0.0 & 0 & 2 & 0.471 & 0.009 & 86.06\% & 99.70\%\\
\rowcolor{lightgray}
Mixup & 0.1 & 0 & 2 & 0.681 & 0.016 & 78.97\% & 98.80\%\\
\rowcolor{lightgray}
Mixup & 0.2 & 0 & 2 & 0.974 & 0.067 & 66.26\% & 96.41\%\\
\rowcolor[gray]{0.9}
CORAL & 0.0 & 0 & 2 & 0.743 & 0.006 & 83.13\% & 97.31\%\\
\rowcolor[gray]{0.9}
CORAL & 0.2 & 0 & 2 & 0.954 & 0.008 & 80.68\% & 97.60\%\\
\rowcolor[gray]{0.9}
CORAL & 0.3 & 0 & 2 & 1.147 & 0.012 & 78.00\% & 96.71\%\\
\rowcolor{lightgray}
MMD & 0.0 & 0 & 2 & 0.776 & 0.005 & 81.42\% & 97.31\%\\
\rowcolor{lightgray}
MMD & 0.1 & 0 & 2 & 1.203 & 0.006 & 74.33\% & 96.41\%\\
\rowcolor{lightgray}
MMD & 0.2 & 0 & 2 & 1.832 & 0.066 & 65.04\% & 93.11\%\\
\rowcolor[gray]{0.9}
RSC & 0.0 & 0 & 2 & 1.089 & 0.003 & 77.75\% & 95.81\%\\
\rowcolor[gray]{0.9}
RSC & 0.1 & 0 & 2 & 2.535 & 0.129 & 63.81\% & 93.41\%\\
\rowcolor[gray]{0.9}
RSC & 0.2 & 0 & 2 & 4.732 & 0.560 & 43.52\% & 82.63\%\\
\rowcolor{lightgray}
VREx & 0.0 & 0 & 2 & 0.593 & 0.002 & 84.84\% & 97.60\%\\
\rowcolor{lightgray}
VREx & 0.1 & 0 & 2 & 0.912 & 0.009 & 77.51\% & 97.01\%\\
\rowcolor{lightgray}
VREx & 0.2 & 0 & 2 & 1.518 & 0.049 & 66.99\% & 94.61\%\\
\rowcolor[gray]{0.9}
MTL & 0.0 & 0 & 2 & 1.269 & 0.001 & 79.95\% & 96.11\%\\
\rowcolor[gray]{0.9}
MTL & 0.2 & 0 & 2 & 2.477 & 0.060 & 67.73\% & 93.41\%\\
\rowcolor{lightgray}
SD & 0.0 & 0 & 2 & 0.589 & 0.113 & 85.33\% & 98.20\%\\
\rowcolor{lightgray}
SD & 0.2 & 0 & 2 & 0.930 & 0.262 & 80.44\% & 97.60\%\\
\rowcolor{lightgray}
SD & 0.3 & 0 & 2 & 1.191 & 0.454 & 73.35\% & 96.11\%\\
\end{tabular}
\label{tab:evaluation_transfer_pacs}
\end{table}

\subsubsection{Office-Home}

We present results from Algorithm \ref{alg:measure_transfer} with information about losses and accuracies, for a wide array of algorithms in Table \ref{tab:evaluation_transfer_OfficeHome} for Office-Home. It can be seen that CORAL and SD learn more robust classifiers while other algorithms are not quite transferable: with a small decrease of source accuracy the target accuracy drops significantly.

\begin{table}[]
\caption{Evaluation of transferability of popular algorithms for domain generalization on  Office-Home. \textbf{algorithm}: the model that we evaluate; \textbf{$\bm{\delta}$}: the adversarial radius $\d$ we choose in Algorithm \ref{alg:measure_transfer_whole}; \textbf{max/min index}: the index of the domain with the maximal/minimal (test) classification errors (w.r.t.~0-1 loss); \textbf{max/min loss}: the largest/smallest loss among domains (including the target domain); \textbf{worst/best acc}: the smallest/largest classification test accuracies among domains (including the target domain). }
\centering
\begin{tabular}{c c c c c c c c c}
\bf algorithm & $\bm{\delta}$ & \bf max index &\bf  min index & \bf max loss & \bf min loss & \bf worst acc & \bf best acc \\
\hline \hline
\rowcolor[gray]{0.9}
ERM & 0.0 & 0 & 2 & 2.688 & 0.054 & 54.43\% & 88.16\%\\
\rowcolor[gray]{0.9}
ERM & 0.1 & 0 & 2 & 3.701 & 0.098 & 47.63\% & 87.37\%\\
\rowcolor{lightgray}
GroupDRO & 0.0 & 0 & 2 & 2.940 & 0.072 & 58.76\% & 88.61\%\\
\rowcolor{lightgray}
GroupDRO & 0.1 & 0 & 2 & 4.042 & 0.147 & 50.72\% & 86.81\%\\
\rowcolor[gray]{0.9}
SagNet & 0.0 & 0 & 2 & 2.030 & 0.055 & 56.08\% & 87.94\%\\
\rowcolor[gray]{0.9}
SagNet & 0.1 & 0 & 2 & 2.316 & 0.071 & 54.02\% & 88.05\%\\
\rowcolor{lightgray}
Mixup & 0.0 & 0 & 2 & 1.657 & 0.051 & 60.62\% & 90.76\%\\
\rowcolor{lightgray}
Mixup & 0.1 & 0 & 2 & 2.074 & 0.075 & 53.40\% & 90.08\%\\
\rowcolor[gray]{0.9}
CORAL & 0.0 & 0 & 2 & 1.878 & 0.043 & 59.79\% & 89.06\%\\
\rowcolor[gray]{0.9}
CORAL & 0.1 & 0 & 2 & 2.111 & 0.053 & 56.70\% & 88.73\%\\
\rowcolor{lightgray}
MMD & 0.0 & 0 & 2 & 2.201 & 0.037 & 56.49\% & 89.74\%\\
\rowcolor{lightgray}
MMD & 0.1 & 0 & 2 & 2.860 & 0.060 & 50.93\% & 88.16\%\\
\rowcolor[gray]{0.9}
VREx & 0.0 & 0 & 2 & 1.926 & 0.207 & 55.46\% & 85.46\%\\
\rowcolor[gray]{0.9}
VREx & 0.1 & 0 & 2 & 2.414 & 0.245 & 49.28\% & 84.89\%\\
\rowcolor{lightgray}
MTL & 0.0 & 0 & 2 & 2.736 & 0.047 & 52.58\% & 87.71\%\\
\rowcolor{lightgray}
MTL & 0.1 & 0 & 2 & 3.921 & 0.109 & 42.06\% & 85.12\%\\
\rowcolor[gray]{0.9}
SD & 0.0 & 0 & 2 & 1.535 & 0.047 & 64.33\% & 91.54\%\\
\rowcolor[gray]{0.9}
SD & 0.1 & 0 & 2 & 1.717 & 0.049 & 63.51\% & 92.33\%\\
\hline
\end{tabular}
\label{tab:evaluation_transfer_OfficeHome}
\end{table}

\end{document}